\documentclass[runningheads,envcountsame]{llncs}
\usepackage[T1]{fontenc}
\usepackage{graphicx}

\usepackage{subcaption}
\usepackage{booktabs}
\usepackage[bookmarks,bookmarksopen,bookmarksdepth=2]{hyperref}

\usepackage{color}

\usepackage{amsmath,amssymb,bm}  
\newcommand*{\matr}[1]{\mathbf{#1}}
\newcommand*{\vect}[1]{\bm{#1}}

\DeclareMathOperator*{\argmax}{arg\,max}

\let\llncsproof\proof
\let\endllncsproof\endproof
\renewenvironment{proof}[1][\proofname]{%
  \ifx#1\proofname
    \llncsproof
  \else
    \llncsproof[#1]
  \fi
}{%
  \qed%
  \endllncsproof%
}

\usepackage{verbatim}

\usepackage[capitalize,noabbrev]{cleveref}

\begin{document}
%
\title{Knowledge Gradient for Multi-Objective Bayesian Optimization with Decoupled Evaluations}
\titlerunning{Knowledge Gradient for Decoupled Objectives}
%
\author{Jack M. Buckingham\inst{1}\and
Sebastian Rojas Gonzalez\inst{2} \and
Juergen Branke\inst{1}}
\authorrunning{J. M. Buckingham et al.}
%
\institute{University of Warwick, Coventry, UK \and
Surrogate Modeling Lab, Gent University, Belgium} 
%
\maketitle              
\begin{abstract}
Multi-objective Bayesian optimization aims to find the Pareto front of  trade-offs between a set of expensive objectives while collecting as few samples as possible. In some cases, it is possible to evaluate the objectives separately, and a different latency or evaluation cost can be associated with each objective. This decoupling of the objectives presents an opportunity to learn the Pareto front faster by avoiding unnecessary, expensive evaluations. We propose a  scalarization based knowledge gradient acquisition function which accounts for the different evaluation costs of the objectives. 
We prove asymptotic consistency of the estimator of the optimum for an arbitrary, D-dimensional, real compact search space and show empirically that the algorithm performs comparably with the state of the art and significantly outperforms versions which always evaluate both objectives.

\end{abstract}
\section{Introduction}
\label{sec:introduction}
Bayesian optimization (BO) is a technique for finding the global maximum of an expensive-to-evaluate objective function (e.g., high computational or financial costs) while taking as few samples as possible \cite{couckuyt2022bayesian,garnett2023bayesopt}. Data is collected either sequentially or in batches, and a Bayesian model, usually a Gaussian process (GP), is used to estimate both the objective function and the uncertainty in that estimate based on the samples collected so far. An acquisition function is then used to trade off between exploring regions of the input space with high uncertainty and exploiting regions which are already known to contain good values. In real-world problems, it is common to have multiple competing objectives. In such problems there is typically no single best solution and instead the aim is to find a set of optimal trade-offs, where improving  one objective necessarily means regressing  another. This set of optimal trade-offs in design space is referred to as the \emph{Pareto set} while its image in objective space is the \emph{Pareto front}.

The use of random scalarizations is a popular technique for solving multi-objective Bayesian optimization problems \cite{knowles2006parego,paria2020scalarizations}. At each step, a randomly chosen scalarization vector is used to convert the multi-objective problem into a single objective problem. By varying the scalarization weights throughout the optimization,  different points on the Pareto front can be eventually discovered \cite{miettinen1999moo}. Instead of using random scalarizations, a more sophisticated approach takes expectation over possible scalarizations, meaning all scalarizations are considered at each step and no time is wasted on scalarizations corresponding to parts of the Pareto front which are already well known \cite{astudillo2017makg}.

Traditionally, multi-objective BO assumes that all objectives will be evaluated at each design vector sampled - so-called \emph{coupled evaluations}. However, in practice this might not always make sense. For example, in turbine design, objectives such as efficiency are calculated using computational fluid dynamics, while mechanical stresses are calculated using finite element analysis \cite{braembussche2008turbomachinery}. Similarly, computing the efficiency of the turbine design may be more expensive than computing the cost of a given configuration. In these problems, where we can make \emph{decoupled evaluations}, there are two reasons we may want to evaluate some objectives more often than others. Firstly, some objectives may be cheaper or faster-to-calculate than others and so it may make sense to evaluate those objectives more frequently. Secondly, some objectives may be harder to learn than other objectives, and consequently these objectives require more evaluations for finding the Pareto front. 


In this paper, we tackle multi-objective optimization problems where (1) objectives can be evaluated separately; (2) objectives share resources (known as \emph{competitive decoupling}); (3) all objectives are expensive to some degree; and (4) some objectives are harder-to-learn than others. Our contributions are as follows: 
\begin{enumerate}
    \item We provide a natural extension of the \emph{multi-attribute knowledge-gradient} acquisition function \cite{astudillo2017makg} to the case where objectives are evaluated separately and have known, constant latencies / costs.
    \item We provide a theoretical guarantee that the algorithm will converge to the global optimum for a decision maker with any linear utility function.
    \item We present experiments in two classes of bi-objective problems, one where the length scales of the objectives differ and one where the variance of the observation noise differs. Our algorithm is competitive with the state-of-the-art, and we demonstrate for all algorithms that exploiting the decoupling of the objectives improves convergence rate.
    \item We demonstrate the importance of taking expectation over scalarizations instead of using a sequence of random scalarizations.
\end{enumerate}

\paragraph{Related Work}
The optimization of multiple, decoupled objectives with different latencies (also referred to as heterogeneous objectives) has received most of its attention from the evolutionary algorithms (EAs) community \cite{allmendinger2015nonuniform,chugh2018hkrvea,blank2022hetero,mamun2022hetero,wang2022hetero}. A recent review is given in \cite{allmendinger2021hetero}.
While EAs are efficient black-box optimizers, they tend to be very data-hungry. On the other hand, BO offers data-efficiency on top of black-box optimization. For instance, Loka et al. \cite{loka2022chvei} directly incorporate evaluation of the cheap objectives into a pair of hypervolume-based acquisition functions for BO. Consequently, they are evaluated many times while the acquisition function is optimized. This is only applicable when the cheap objective is extremely cheap, as is the case for an objective with a known analytical formula.

It is common to consider the case of decoupled objectives in information based acquisition functions \cite{hernandezlobato2016pesmo,hernandezlobato2016pesc,suzuki2020pfes}. These algorithms attempt to maximize mutual information between the Pareto front or Pareto set and the proposed observation. Another acquisition function known as the hypervolume knowledge gradient (HVKG) \cite{daulton2023hvkg} chooses the next sample to maximize the increase in the hypervolume associated with the posterior mean of the GP surrogate model. Both knowledge gradient and information based acquisition functions can naturally be extended to partial observations such as the case of decoupled objectives.


\section{Problem definition}
\label{sec:background}
Let \(\vect{f}^* : \mathcal{X} \to \mathbb{R}^M\) be an unknown multi-objective function on a space \(\mathcal{X} \subset \mathbb{R}^D\), whose components can be evaluated separately.\footnote{We will use \(\vect{f}^*\) to denote the unknown, true function and \(\vect{f}\) to denote the Gaussian process (GP) model.}
Let \(\{u_{\vect{\lambda}}: \vect{\lambda} \in \Lambda\}\) be a set of monotonic utility functions \(u_{\vect{\lambda}} : \mathbb{R}^M \to \mathbb{R}\), each representing a possible decision maker.
Since the utility functions are monotonic, the maximizers will be Pareto optimal solutions.
We aim to find a finite subset \(S \subset \mathcal{X}\) of Pareto optimal solutions such that the expected utility of a random decision maker choosing from \(S\) is maximized. That is, we solve
\begin{equation} \label{eq:problem-definition}
    S^* \in \argmax_{S \subset \mathcal{X},\; |S| \leq N_S} \mathbb{E}_{\vect{\lambda} \sim p(\vect{\lambda})} \left[ \max_{\vect{x} \in S} u_{\vect{\lambda}}(\vect{f}^*(\vect{x})) \right]
\end{equation}
for some \(N_S \in \mathbb{N}\).

Suppose we can make observations of \(\vect{f}^*\) according to the model
\begin{equation} \label{eq:obs-model-decoupled}
    y_m = f_m^*(\vect{x}) + \varepsilon_m, \quad \varepsilon_m \sim \mathcal{N}(0, \sigma_m^2)
\end{equation}
where \(m \in \{1, \dots, M\}\) is the component of \(\vect{f}^*\) being evaluated, \(\vect{x} \in \mathcal{X}\) is the input and \(\varepsilon_m\) represents observation noise. Suppose further that the costs, \(c_1, \dots, c_M\), of evaluating each objective are constant and known.

The problem is to iteratively choose sample locations \(\vect{x}_1, \vect{x}_2, \dots \in \mathcal{X}\) and objectives \(m_1, m_2, \dots \in \{1, \dots, M\}\) to be evaluated, from which approximate solutions \(S_1, S_2, \dots\) to Equation~\eqref{eq:problem-definition} can be derived, which converge quickly to the optimal expected utility in terms of the total sampling cost, \(C_N = \sum_{n = 1}^N c_{m_n}\).

\section{Background}

\paragraph{Bayesian Optimization}
Bayesian optimization uses two ingredients: a probabilistic model of the objective and an acquisition function which is maximized to give the next sample location.

For the probabilistic model, a multi-output GP prior distribution \cite{alvarez2011multioutputgp} is placed on the set of possible objectives, and is typically paired with an observation model which assumes additive Gaussian noise
\begin{equation} \label{eq:obs-model-coupled}
    \vect{y} = \vect{f}(\vect{x}) + \vect{\varepsilon}, \quad \vect{\varepsilon} \sim \mathcal{N}(0, \Sigma).
\end{equation}
Typically the covariance matrix, \(\Sigma\), is assumed to be diagonal, meaning the noise added to the different objectives is independent.

The multi-output GP, \(\vect{f} \sim \mathcal{GP}(\vect{\mu}, \matr{K})\) is fully characterized by its vector-valued \emph{mean function} \(\vect{\mu} : \mathcal{X} \to \mathbb{R}^M\) and positive semi-definite matrix-valued \emph{covariance function} or \emph{kernel} \(\matr{K}: \mathcal{X} \times \mathcal{X} \to \mathbb{R}^{M \times M}\).
In this paper, we take \(\matr{K}\) to be diagonal-valued which, paired with a diagonal \(\Sigma\), is equivalent to modeling the objectives with independent, single-output Gaussian processes, \(f_m \sim \mathcal{GP}(\mu_m, k_{m,m})\).

When conditioned on \(N\) observations \(y_1, \dots, y_N \in \mathbb{R}\) at inputs \(\vect{x}_1, \dots, \vect{x}_N \in \mathcal{X}\) with additive Gaussian noise of variance \(\sigma^2\), the posterior distribution of a single-output GP, \(f \sim \mathcal{GP}(\mu, k)\), is another Gaussian process with mean function \(\mu_N : \mathcal{X} \to \mathbb{R}\) and covariance function \(k_N: \mathcal{X} \times \mathcal{X} \to \mathbb{R}\) given by \cite{rasmussenwilliams2006gps}
\begin{subequations} \label{eq:gp-posterior}
\begin{align}
    \mu_N(\vect{x}) &= \mu(\vect{x}) + k(\vect{x}, \matr{X}) \bigl( k(\matr{X}, \matr{X}) + \sigma^2 \bigr)^{-1} \bigl(\vect{y} - \mu(\matr{X})\bigr), \label{eq:gp-posterior-mean}\\
    k_N(\vect{x}, \vect{x}') &= k(\vect{x}, \vect{x}') - k(\vect{x}, \matr{X}) \bigl( k(\matr{X}, \matr{X}) + \sigma^2 \bigr)^{-1} k(\matr{X}, \vect{x}'). \label{eq:gp-posterior-covar}
\end{align}
\end{subequations}
Here we have written \(\matr{X}\) for the matrix formed by stacking the observation locations \(\vect{x}_1, \dots, \vect{x}_N\) and used the \(\mu(\matr{X}) \in \mathbb{R}^N\) and \(k(\matr{X}, \matr{X}) \in \mathbb{R}^{N \times N}\) to denote batch evaluations. For a full introduction to BO, see \cite{shahriari2016bayesopt,frazier2018botutorial,garnett2023bayesopt}.

\paragraph{Multi-Attribute Knowledge Gradient (maKG)}
The second ingredient of Bayesian optimization is the acquisition function.
The principle underlying the multi-attribute knowledge gradient acquisition function introduced in \cite{astudillo2017makg} is that, once the evaluation budget has been reached, for a given \(\vect{\lambda} \in \Lambda\) we will recommend the input \(\vect{x}^* \in \mathcal{X}\) which maximizes the posterior mean of the utility of the GP model.

More formally, let \(\mu_N(\vect{x}';\vect{\lambda}) = \mathbb{E}_N[u_{\vect{\lambda}}(\vect{f}(\vect{x}'))]\) denote the posterior mean of the utility of \(\vect{f}(\vect{x}')\) under utility function \(u_{\vect{\lambda}}\) conditioned on the value of \(\vect{f}\) at \(\vect{x}_1, \dots, \vect{x}_N\).
During the optimization, let \(\vect{x} \in \mathcal{X}\) denote a potential next sample location, and let \(\vect{y}^{\vect{x}} \in \mathbb{R}^M\) be the random variable for the observation we might make at this location under the model \eqref{eq:obs-model-coupled}.
Importantly at this point, we are assuming all \(M\) objectives are observed.
After applying a utility function \(u_{\vect{\lambda}}\), denote the posterior mean at some other location \(\vect{x}' \in \mathcal{X}\) as \(\mu_{N+}(\vect{x}'; \vect{x}, \vect{\lambda}) = \mathbb{E}_N\left[ u_{\vect{\lambda}}(\vect{f}(\vect{x}')) \,\middle|\, \vect{y}^{\vect{x}} \right]\). Subtracting the maximum of the current posterior mean of \(\vect{f}\) and taking expectation over \(\vect{y}^{\vect{x}}\) gives the \emph{knowledge gradient} for utility \(u_{\vect{\lambda}}\),
\begin{equation} \label{eq:scalarized-knowledge-gradient}
    \alpha_\mathrm{KG}(\vect{x}; \vect{\lambda}) = \mathbb{E}_N\left[ \max_{\vect{x}' \in \mathcal{X}} \mu_{N+}(\vect{x}'; \vect{x}, \vect{\lambda}) \right] - \max_{\vect{x}' \in \mathcal{X}} \mu_N(\vect{x}'; \vect{\lambda}).
\end{equation}

One simple BO approach is to randomly select a new parameter \(\vect{\lambda} \in \Lambda\) for the utility function \(u_{\vect{\lambda}}\), at each iteration and optimize  \(\alpha_\mathrm{KG}(\cdot,\, \vect{\lambda})\). An alternative is to instead take expectation over scalarization weights of an acquisition function applied to the scalarized objective. This approach was taken in \cite{astudillo2017makg} for the multi-attribute knowledge gradient.  In formulae,
\begin{equation} \label{eq:multi-attribute-knowledge-gradient}
    \alpha_\mathrm{maKG}(\vect{x}) = \mathbb{E}_{\vect{\lambda} \sim p(\vect{\lambda})} \bigl[ \alpha_\mathrm{KG}(\vect{x}; \vect{\lambda}) \bigr].
\end{equation}
The authors restrict to linear utility functions, \(u_{\vect{\lambda}}(\vect{f}(\vect{x}')) = \vect{\lambda} \cdot \vect{f}(\vect{x}')\). This has the advantage that the linear utility function commutes with expectation. 

Both when using random scalarizations and expectation over scalarizations, there is a choice to be made for the distribution \(p(\vect{\lambda})\) of \(\vect{\lambda}\). Paria et al. \cite{paria2020scalarizations} use this to encode the preferences of the decision maker for learning about different areas of the Pareto front. Here we concentrate on the case of linear scalarizations where \(p(\vect{\lambda})\) is a uniform distribution over the standard simplex in \(\mathbb{R}^M\).


\section{Cost Weighted Multi-Objective Knowledge Gradient}
\label{sec:mokg}




In order to extend maKG to multiple objectives, we need simply observe that the multi-output GP can be conditioned on observations of just one coordinate of \(\vect{f}\) at a time. 
In the decoupled setting, we have observations \(y_{1,m_1}, \dots, y_{N,m_N}\) with \(y_{n,m_n} = f_{m_n}(\vect{x}_n) + \varepsilon_{n,m_n}\) for each \(n\) as in Equation~\eqref{eq:obs-model-decoupled}.

For a given \(\vect{\lambda} \in \Lambda\), write \(\mu_N(\vect{x}'; \vect{\lambda}) = \mathbb{E}_N[u_{\vect{\lambda}}(\vect{f}(\vect{x}'))]\) for the posterior mean of \(\vect{f}\) conditional on the \(N\) observations.
Let \(\vect{x} \in \mathcal{X}\) denote a potential next sample location, let \(m \in \{1, \dots, M\}\) be a potential next objective and let \(y_m^{\vect{x}} = f_m(\vect{x}) + \varepsilon_m\) be the random variable for the corresponding potential observation.
Write \(\mu_{N+}(\vect{x}'; \vect{x}, m, \vect{\lambda}) = \mathbb{E}_N[ u_{\vect{\lambda}}(\vect{f}(\vect{x}')) \,|\, y_m^{\vect{x}} ]\) for the posterior mean conditional on \(y_m^{\vect{x}}\) and the \(N\) real observations.
Subtracting the maxima of these two functions and taking expectation over \(y_m^{\vect{x}}\) gives the \emph{multi-objective knowledge gradient},
\begin{equation} \label{eq:mokg}
    \alpha_\mathrm{MOKG}(\vect{x}, m; \vect{\lambda}) = \mathbb{E}_N\left[ \max_{\vect{x}' \in \mathcal{X}} \mu_{N+}(\vect{x}'; \vect{x}, m, \vect{\lambda}) \right] - \max_{\vect{x}' \in \mathcal{X}} \mu_N(\vect{x}'; \vect{\lambda}).
\end{equation}
This is the result for a single utility \(u_{\vect{\lambda}}\). Taking expectation over \(\vect{\lambda} \sim p(\vect{\lambda})\) gives
\begin{equation}\label{eq:mokg-e}
    \overline{\alpha}_\mathrm{MOKG}(\vect{x}, m) = \mathbb{E}_{\vect{\lambda} \sim p(\vect{\lambda})}\left[ \alpha_\mathrm{MOKG}(\vect{x}, m;\, \vect{\lambda}) \right].
\end{equation}

In the context of objectives with different evaluation costs, we divide the MOKG by the cost to evaluate the proposed objective to give a value-per-unit-cost. This is the approach taken by Snoek et al. \cite{snoek2012practical} in the single objective case. We obtain our main contribution in this work, the \emph{cost-weighted multi-objective knowledge gradient (C-MOKG)} acquisition function,
\begin{subequations} \label{eq:mokg-costs}
\begin{align}
    \alpha_\mathrm{C\text{-}MOKG}(\vect{x}, m; \vect{\lambda}, \vect{c}) = \frac{1}{c_m} \alpha_\mathrm{MOKG}(\vect{x}, m; \vect{\lambda}), \label{eq:mokg-costs-r} \\
    \overline{\alpha}_\mathrm{C\text{-}MOKG}(\vect{x}, m; \vect{c}) = \frac{1}{c_m} \overline{\alpha}_\mathrm{MOKG}(\vect{x}, m), \label{eq:mokg-costs-e}
\end{align}
\end{subequations}
where \(\vect{c}\) is the vector of costs associated with each of the \(M\) objectives.


\subsection{Efficient Calculation and Optimization} \label{subsec:calculation}

We will first present a method for calculating and optimizing C-MOKG in the case of random scalarizations using a discrete approximation. We will then use this method alongside a quasi-Monte-Carlo approximation to calculate and optimize C-MOKG when taking expectation over scalarizations.

As in \cite{astudillo2017makg}, we now restrict to linear utility functions, \(u_{\vect{\lambda}}(\vect{f}(\vect{x}')) = \vect{\lambda} \cdot \vect{f}(\vect{x}')\) in order to exploit the fact that expectation commutes with linear functions.

\paragraph{Discrete Approximation} Early work which introduced the knowledge gradient \cite{frazier2008kg,frazier2009kg} focused on discrete search spaces. This has inspired a common computational strategy for knowledge gradient in low input dimensions \cite{wu2016paralleldiscretekg,pearce2017robust,pearce2018multitask}, where a discrete approximation for the input space, \(\mathcal{X}\), is used in the inner optimization, while retaining the full continuous space for the proposed next sample location. Concretely, let \(\mathcal{X}_\mathrm{disc}\) be a finite subset approximating \(\mathcal{X}\). Then
\begin{multline} \label{eq:mokg-discrete}
    \alpha_\mathrm{MOKG}(\vect{x}, m; \vect{\lambda}) \approx \hat{\alpha}_\mathrm{MOKG}(\vect{x}, m; \vect{\lambda}) = \\
    \mathbb{E}_N\left[ \max_{\vect{x}' \in \mathcal{X}_\mathrm{disc}} \mu_{N+}(\vect{x'}; \vect{x}, m, \vect{\lambda}) \right]
    - \max_{\vect{x}' \in \mathcal{X}_\mathrm{disc}} \mu_N(\vect{x}'; \vect{\lambda}).
\end{multline}

Since expectation commutes with linear operators, we can use Equation~\eqref{eq:gp-posterior-mean} to write \(\mu_{N+}(\vect{x}'; \vect{x}, m, \vect{\lambda})\) as an affine function of the hypothesized observation \(y_m^{\vect{x}}\),
\begin{equation} \label{eq:mokg-discrete-2}
    \mu_{N+}(\vect{x}'; \vect{x}, m, \vect{\lambda}) = 
    \vect{\lambda} \cdot \tilde{\vect{\mu}}(\vect{x}')
    + \vect{\lambda} \cdot \tilde{\matr{K}}_{:,m}(\vect{x}', \vect{x})\frac{y_m^{\vect{x}} - \tilde{\mu}_m(\vect{x})}{\tilde{k}_{m,m}(\vect{x}, \vect{x}) + \sigma_m^2}.
\end{equation}
Here, \(\tilde{\vect{\mu}}\) and \(\tilde{\matr{K}}\) are the posterior mean and covariance functions of \(\vect{f}\) conditional on the \(N\) observations so far.
It is therefore possible to efficiently calculate \(\hat{\alpha}_\mathrm{MOKG}(\vect{x}, m;\, \vect{\lambda})\) using Algorithm~2 in \cite{frazier2009kg}.
Furthermore, the resulting analytical expression is deterministic and its derivatives with respect to the candidate input \(\vect{x}\) can be found with automatic differentiation. Therefore, we can use a deterministic gradient based optimizer such as multistart L-BFGS-B to find the global maximum.

We can then optimize C-MOKG over both \(m\) and \(\vect{x}\) by optimizing the result for every \(m\) and choosing the largest. Indeed, writing \([M] = \{1, \dots, M\}\),
\begin{equation} \label{eq:cmokg-discrete}
    \max_{\substack{\vect{x} \in \mathcal{X} \\ m \in [M]}} \hat{\alpha}_\mathrm{C\text{-}MOKG}(\vect{x}, m; \vect{\lambda}, \vect{c})
    = \max_{m \in [M]} \frac{1}{c_m} \max_{\vect{x} \in \mathcal{X}} \hat{\alpha}_\mathrm{MOKG}(\vect{x}, m; \vect{\lambda}).
\end{equation}

When optimizing knowledge gradient acquisition functions, it is common to neglect the second term in Equation~\eqref{eq:mokg-discrete} which is constant with respect to \(\vect{x}\). However, it is important that we do not neglect this term when performing the outer maximization in Equation~\eqref{eq:cmokg-discrete}, since the factor \(1/c_m\) means that it is no longer constant with respect to \(m\).

\paragraph{Quasi-Monte-Carlo}
The discrete approximation from the previous section is sufficient to optimize C-MOKG in the case of a single scalarization. We can extend this to a way to optimize C-MOKG with expectation over scalarizations using a (quasi-)Monte-Carlo approximation (qMC). Indeed,
\begin{equation} \label{eq:cmokg-e-discrete}
    \overline{\alpha}_{\text{C-MOKG}}(\vect{x}, m)
    \approx \frac{1}{Q} \sum_{j=1}^{Q} \hat{\alpha}_{\text{C-MOKG}}(\vect{x}, m; \vect{\lambda}^{(j)}).
\end{equation}
Here \(\vect{\lambda}^{(1)}, \dots, \vect{\lambda}^{(Q)}\) is a qMC sample of size \(Q\).
Each term in the sum can be calculated using the discretization technique from the previous section and the average can be optimized with multistart L-BFGS-B.

\subsection{Theoretical Results}
Let \(\Lambda \subset \mathbb{R}^M\) denote the standard simplex.
The following results hold for C-MOKG defined using linear utility functions, \(u_{\vect{\lambda}}(\vect{f}(\vect{x}')) = \vect{\lambda} \cdot \vect{f}(\vect{x}')\) for \(\vect{\lambda} \in \Lambda\).

Our first result establishes that the cost-aware multi-objective knowledge gradient is everywhere non-negative. This is a standard result for knowledge-gradient acquisition functions and is the reason that we do not need to take the positive part inside the expectation as is necessary with expected improvement.
\begin{lemma} \label{thm:cmokg-nonneg}
    Both forms of the cost-aware multi-objective knowledge gradient are non-negative. That is, for all \(\vect{x} \in \mathcal{X}\), \(m \in \{1, \dots, M\}\) and all \(\vect{\lambda} \in \Lambda\),
    \begin{equation*}
        \alpha_\mathrm{C\text{-}MOKG}(\vect{x}, m;\, \vect{\lambda}) \geq 0
        \quad\text{and}\quad
        \overline{\alpha}_\mathrm{C\text{-}MOKG}(\vect{x}, m) \geq 0,
    \end{equation*}
    almost surely.
\end{lemma}
This is a consequence of the maximum of the expectation of a stochastic process being at most the expectation of the maximum of that process, and is proved in \cref{sec:proofs}.

Our main theoretical contribution ensures that when choosing samples with C-MOKG using either expectation over scalarizations, or random scalarizations, the scalarized objective values associated with the recommendations of the algorithm will converge to the optimal value.
We assume no model mismatch by dropping the distinction between \(\vect{f}\) and \(\vect{f}^*\).
In particular, since we are considering \(\vect{f}^*\) to be a GP here rather than a function, the following result should be interpreted as a statement about all possible \(\vect{f}^*\) together rather than for any individual sample.

For each \(N \in \mathbb{N}_0\) and each preference vector \(\vect{\lambda} \in \Lambda\), let
\begin{equation} \label{eq:kg-recommendations}
    \vect{x}_{N, \vect{\lambda}}^* \in \argmax_{\vect{x} \in \mathcal{X}} \mathbb{E}_N[\vect{\lambda} \cdot \vect{f}(\vect{x})]
\end{equation}
be a random variable which maximizes the posterior mean of the scalarized objective at stage \(N\).
Thus, \(\vect{\lambda} \cdot \vect{f}(\vect{x}_{1, \vect{\lambda}}^*),\, \vect{\lambda} \cdot \vect{f}(\vect{x}_{2, \vect{\lambda}}^*), \dots\) is the sequence of (noiseless) scalarized objective values we would obtain if we were to use the recommended point at each stage of the optimization. The following theorem tells us that this sequence converges to the true maximum of the scalarized, hidden objective function, \(\vect{\lambda} \cdot \vect{f}\).
\begin{theorem}[Consistency of C-MOKG] \label{thm:consistency}
    Suppose \(\mathcal{X} \subset \mathbb{R}^D\) is compact and define the \(\vect{x}_{n, \vect{\lambda}}^*\) as in Equation~\eqref{eq:kg-recommendations}.
    When using C-MOKG with either random scalarizations, or expectation over scalarizations, we have
    \[\forall \vect{\lambda} \in \Lambda,\quad \vect{\lambda} \cdot \vect{f}(\vect{x}_{N, \vect{\lambda}}^*) \to \max_{\vect{x} \in \mathcal{X}} \vect{\lambda} \cdot \vect{f}(\vect{x}) \quad\text{as}\quad N \to \infty\]
    almost surely and in mean.
\end{theorem}

The proof of this result is based on the work by Bect et al. \cite{bect2009sur}. It proceeds by showing that \(\alpha_{\text{C-MOKG}}\) converges to zero for every choice of \(\vect{\lambda}\) and uses this to prove that the posterior mean converges to the true objective function (possibly up to a constant). \cref{thm:consistency} then follows easily.
We refer the reader to \cref{sec:proofs} for the proofs of these results.

\section{Experiments}
\label{sec:experiments}
The two main reasons for evaluating one objective more frequently than  another are because it is relatively cheap, and because it is harder to learn. Thus, the largest improvement is seen when the expensive objectives are easier to learn. The experiments on synthetic bi-objective problems in this section demonstrate this in two cases, where the cheaper objective is made harder-to-learn using a shorter length scale and the presence of observation noise, respectively.

\paragraph{Synthetic Problems}

We test the algorithm on two families of 100 test problems, each with two input dimensions and two objectives. In order to avoid model mismatch, the objectives are generated independently as samples from different Gaussian processes using a Mat\'ern-\(5/2\) kernel. In the first family of test problems, the first objective has a length scale of 0.2 while the second has a length scale of 1.8. This difference makes the first objective much harder to learn. In this family, no observation noise is added when sampling the problem. In the second family of test problems, both objectives have a length scale of 0.4 and the first is instead made harder to learn by the inclusion of observation noise. The noise added has a standard deviation of 1 which is reasonably large compared with the output scale which is also 1. The second objective is noise free. In both cases, we pretend that the cost or latency of the first objective is 1, while that of the second objective is 10. A full description of the hyper-parameters used to generate these test problems can be found in \cref{sec:experiment-details}.

\paragraph{Bayesian Regret Performance Metric}
In light of the problem definition in Equation~\eqref{eq:problem-definition}, the natural metric is a variant of the often used R2 performance metric for multi-objective optimization \cite{hansen1994evaluating,tu2023r2optim}.
This metric assumes a parameterized utility function with a known distribution of the parameter. We assume $u_{\vect{\lambda}}(\vect{f}^*(\vect{x})) = \vect{\lambda} \cdot \vect{f}^*(\vect{x})$ with $\vect{\lambda}$ uniformly distributed on the standard simplex. The quality of a solution set $S$ is then the expected utility, i.e., 
\begin{equation}
R2(S)=\mathbb{E}_{\vect{\lambda}\sim p(\vect{\lambda})} \left[\max_{\vect{x}\in S}\vect{\lambda}\cdot \vect{f}^*(\vect{x}) \right].
\end{equation}
Our algorithm returns a posterior mean prediction for each objective, \(\vect{f}\) as an estimate of the true function \(\vect{f}^*\).
From this, an approximation \(S_N \subset \mathcal{X}\) of 1000 points in the predicted Pareto set is derived using NSGA-II \cite{deb2002nsgaii}.
A decision maker with a particular utility function defined by $\vect{\lambda}$ would select a solution \(\vect{x}^*_{N,\vect{\lambda}} \in \argmax_{\vect{x} \in S_N} \mathbb{E}_N[\vect{\lambda} \cdot \vect{f}(\vect{x})]\).
However, as the resulting Pareto front is based on \emph{predicted} values, the selected solution from this Pareto front doesn't necessarily obtain the utility they hoped for. This can be accounted for by only recording the true utility of the selected solution, \(\vect{\lambda} \cdot \vect{f}^*(\vect{x}^*_{N,\vect{\lambda}})\).
Subtracting this from the maximum possible utility for any solution \(\vect{x} \in \mathcal{X}\) and taking expectation over \(\vect{\lambda}\) gives the \emph{Bayesian regret},
\begin{equation} \label{eq:bayes-regret}
    R_N = \mathbb{E}_{\vect{\lambda} \sim p(\vect{\lambda})}\left[ \max_{\vect{x} \in \mathcal{X}} \vect{\lambda} \cdot \vect{f}^*(\vect{x}) - \vect{\lambda} \cdot \vect{f}^*(\vect{x}_{N,\vect{\lambda}}^*) \right].
\end{equation}
This is similar to the construction of Bayes regret in \cite{paria2020scalarizations}, but allows for a mismatch between the model and true function.

In practice, we estimate this expectation using a qMC (Sobol' \cite{sobol1967sobolseq}) sample \((\vect{\lambda}^{(j)})_{j=1}^{N_{\vect{\lambda}}}\) of size \(N_{\vect{\lambda}} = 1024\).
The maximization over \(\vect{x} \in \mathcal{X}\) is computed by first generating an approximation \(\hat{\mathcal{X}}_\text{Pareto}^*\) of 1000 points in the Pareto set for \(\vect{f}^*\) using NSGA-II, then selecting the largest.
In summary, we compute
\begin{equation}
    R_N \approx \hat{R}_N = 
    \frac{1}{N_{\vect{\lambda}}} \sum_{j=1}^{N_{\vect{\lambda}}} \left( \max_{\vect{x} \in \mathcal{X}_\text{Pareto}^*} \vect{\lambda}^{(j)} \cdot \vect{f}^*(\vect{x}) - \vect{\lambda}^{(j)} \cdot \vect{f}^*(\vect{x}_{N,\vect{\lambda}}^*) \right).
\end{equation}

\paragraph{Hypervolume performance metric}
The Bayesian regret metric defined using linear utility functions only measures performance on the intersection of the Pareto front with its convex hull.
This could be addressed by using a Pareto compliant utility function \(u_{\vect{\lambda}}\), but a common alternative is to measure the hypervolume enclosed between a reference point and the image of the estimated Pareto set, \(\vect{f}^*(S_N)\).
The hypervolume is calculated using the Dominated Partitioning implementation in BoTorch \cite{balandat2020botorch,lacour2017boxdecomp}.
For the reference point, we use the minimum in each dimension of 1000 points on the Pareto front generated with NSGA-II, minus 1\% of the range in the Pareto front in each dimension.

\subsection{Experimental Details}
For each of the two families of test problem, we run the BO 100 times and present the mean of the Bayesian regret. Each repeat uses a different, independently sampled instance of the test problem and a different initial sample of six points generated from a scrambled Sobol' sequence \cite{sobol1967sobolseq}. The initial points are all evaluated on both objectives.

We compare our algorithm against the hypervolume knowledge gradient (HVKG) \cite{daulton2023hvkg} and a modification of the lower bound approximation to the joint entropy search (JES-LB) for decoupled objectives \cite{tu2022jes}. HVKG is state-of-the-art for problems of decoupled objectives. While the authors of JES-LB did not consider the decoupled case, it has been historically considered in entropy-based acquisition functions \cite{hernandezlobato2016pesmo,hernandezlobato2016pesc,suzuki2020pfes} and JES-LB is state-of-the-art among these.

We further compare to a benchmark algorithm which uses the multi-attribute knowledge gradient (maKG) \cite{astudillo2017makg}. This always evaluates both objectives but is otherwise identical to our algorithm.
For both our acquisition function and the maKG benchmark, we use a uniform \(11 \times 11\) grid for the discretization \(\mathcal{X}_\mathrm{disc}\).
The expectation over scalarizations in \eqref{eq:cmokg-e-discrete} is calculated using a qMC estimate using a scrambled Sobol' sequence of \(Q = 16\) points of dimension \(D-1\), which is randomly regenerated at each BO iteration.

For the surrogate model, we use a Mat\'ern-\(5/2\) kernel like the test problem, however, its hyper-parameters are fitted to the observed data using maximum a posteriori estimates. For the first family of test problems, suggestive priors are placed on the length scales to hint to the model that the first objective has a shorter length scale. As observed previously, the algorithm works best when the cheaper objectives are harder to learn. If an engineer knows that an objective has a shorter length scale or noisier observations, they can incorporate this using the prior distribution on the relevant hyper-parameters. Full details of the prior distributions used are included in \cref{sec:experiment-details}.
The same hyper-priors are used for all acquisition functions compared.

\subsection{Results}
\cref{fig:results-bayesregret} shows the evolution of the mean Bayesian regret over the 100 repeats of the experiment in the two families of test problem. The shaded area shows a 95\% marginal confidence interval in the expected value across the family of test problems, calculated as two standard errors in the mean.
In both cases, the decoupled algorithms outperform the coupled ones because they can save time by skipping samples of the slower, easier-to-learn objective. C-MOKG outperforms HVKG and JES-LB because it is tailored to decision makers with linear utility functions.
However, the significant levels of noise used in the second family of problems means that convergence is much slower for both algorithms than for the first family.

\begin{figure*}[htb]
    \centering
    \begin{subfigure}[b]{0.49\textwidth}
        \centering
        \includegraphics[width=\textwidth]{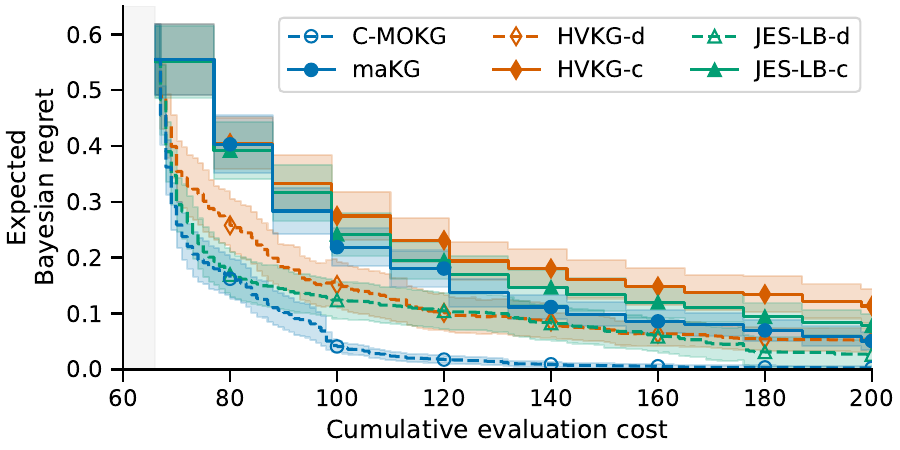}
        \caption{Different length scales}
    \end{subfigure}
    \begin{subfigure}[b]{0.49\textwidth}
        \centering
        \includegraphics[width=\textwidth]{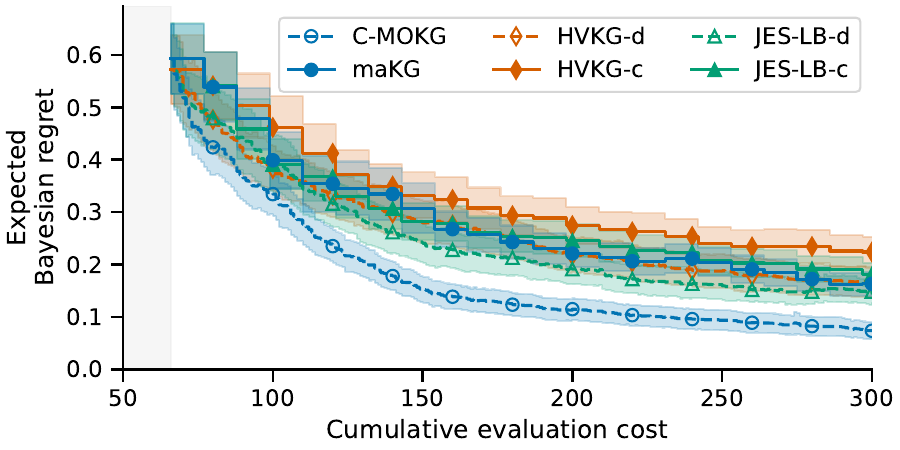}
        \caption{Different observation noise}
    \end{subfigure}
    \caption{A comparison of the evolution of the expected Bayesian regret between C-MOKG/maKG (blue circles), HVKG (red diamonds) and JES-LB (green triangles). For HVKG and JES-LB, a suffix of `-c' or `-d' is used to distinguish the coupled and decoupled algorithms, respectively. The coupled versions are solid while the decoupled versions are dashed. The shaded areas show a 95\% confidence interval in the expected value across the family of test problems (two standard errors around the mean).}
    \label{fig:results-bayesregret}
\end{figure*}

We used linear utility functions throughout this work. However, it is informative to observe that C-MOKG remains competitive when viewed using the  hypervolume regret metric, as shown in \cref{fig:results-hypervol}. This is of particular note because the test problems do not in general have convex Pareto fronts. It is well-known that linear scalarizations are not Pareto compliant for non-convex Pareto fronts. However, the data collected using C-MOKG can be used to inform all parts of the Pareto front, not just the intersection with the convex hull, and we believe this to be the explanation for the algorithm's strong performance. In Figure \ref{fig:pareto-front-convergence-example-3} we show an example of the C-MOKG algorithm converging to the Pareto front.

\begin{figure*}[htb]
    \centering
    \begin{subfigure}[b]{0.49\textwidth}
        \centering
        \includegraphics[width=\textwidth]{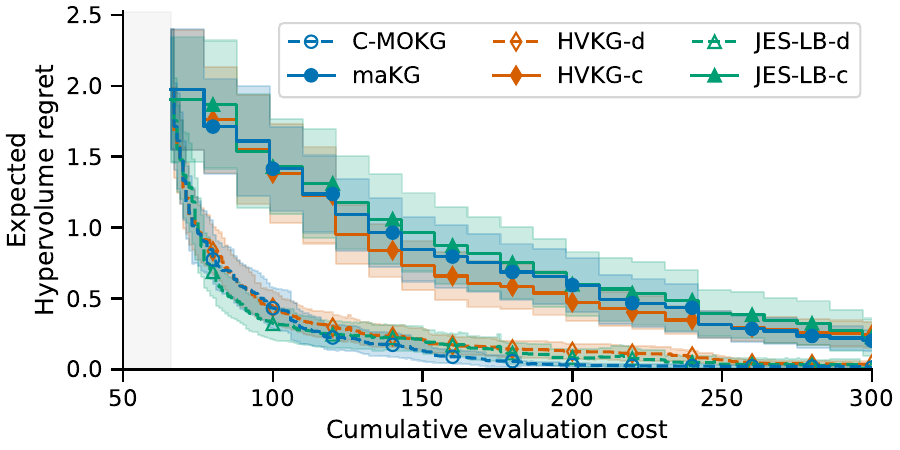}
        \caption{Different length scales}
    \end{subfigure}
    \begin{subfigure}[b]{0.49\textwidth}
        \centering
        \includegraphics[width=\textwidth]{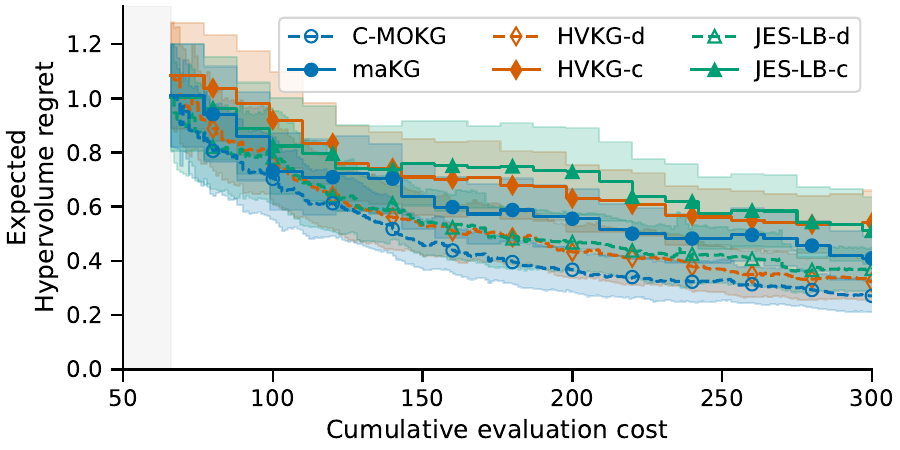}
        \caption{Different observation noise}
    \end{subfigure}
    \caption{Results from the experiments in \cref{fig:results-bayesregret}, using hypervolume regret.}
    \label{fig:results-hypervol}
\end{figure*}

\begin{figure}[h!]
    \centering
    \begin{subfigure}[b]{0.32\linewidth}
        \centering
        \includegraphics[width=\linewidth]{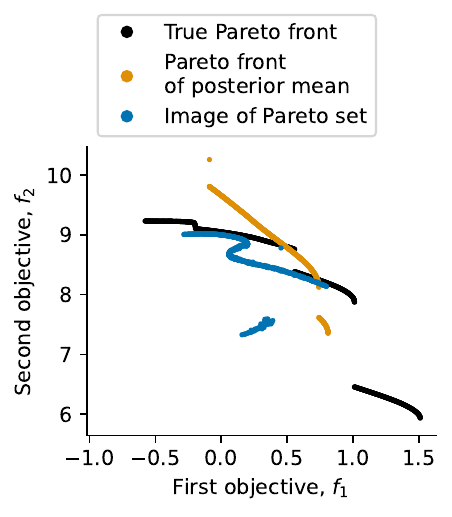}
        \caption{Initial samples only}
        \label{fig:pareto-front-convergence-example-0}
    \end{subfigure}
    \begin{subfigure}[b]{0.32\linewidth}
        \centering
        \includegraphics[width=\linewidth]{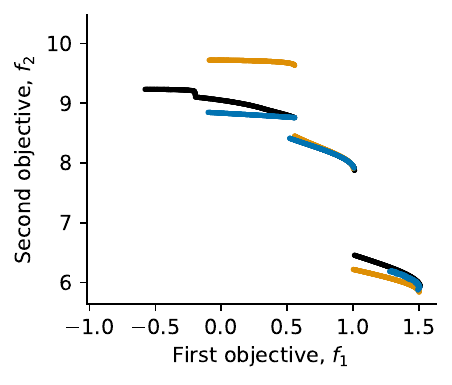}
        \caption{30 BO samples}
        \label{fig:pareto-front-convergence-example-2}
    \end{subfigure}
    \begin{subfigure}[b]{0.32\linewidth}
        \centering
        \includegraphics[width=\linewidth]{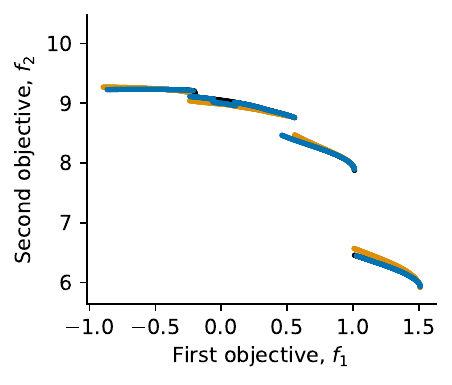}
        \caption{70 BO samples}
        \label{fig:pareto-front-convergence-example-3}
    \end{subfigure}
    \caption{Convergence of the GP surrogate as samples are collected with C-MOKG.
    }
    \label{fig:pareto-front-convergence-example}
\end{figure}


\subsection{Comparison to Random Scalarizations}
We have claimed throughout that taking expectation over scalarizations is particularly beneficial when objectives are evaluated separately.
In this section, we compare C-MOKG taking expectation over scalarizations (as in \cite{astudillo2017makg}), with the version which uses random scalarizations (as in ParEGO \cite{knowles2006parego}). To generate the random scalarization weights, at each iteration we take the next element from a scrambled Sobol' sequence, which ensures that weights are well spread over the simplex.

The results in \cref{fig:results-expectation-vs-random} show a big difference between using expectation over scalarizations (blue circles) and using random scalarizations (orange triangles) in the decoupled case (dashed). The difference is not that large in the coupled case when objectives are evaluated together (solid).
The reason is probably that some of the random scalarizations will overly favor the slow, easy-to-learn objective. Therefore, the algorithm wastes time taking samples in order to learn about a part of the Pareto front which is already well known. Conversely, the algorithm which takes expectation over scalarizations can look at potential improvement across the whole Pareto front and sample the faster, harder-to-learn objective more often.

\begin{figure}[t]
    \centering
    \includegraphics[width=0.5\linewidth]{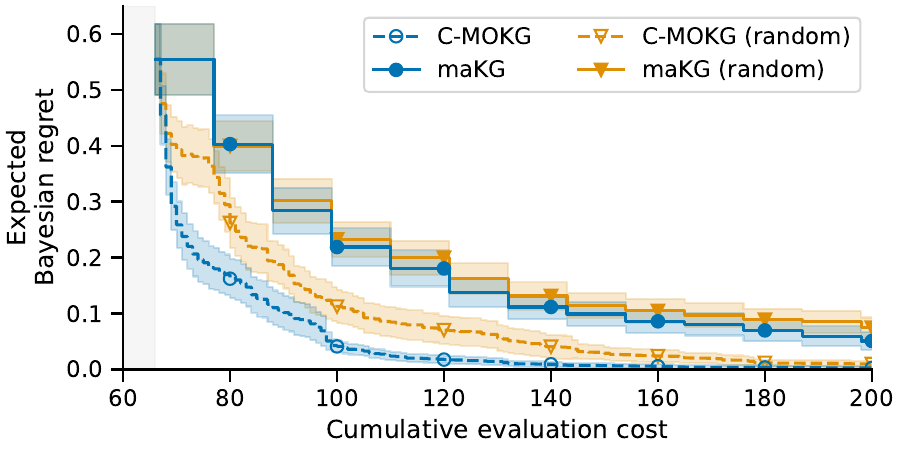}
    \caption{A comparison of the expected Bayesian regret when using random scalarizations (orange triangles) and expectation over scalarizations (blue circles), for the first family of test problems. The decoupled C-MOKG (dashed), significantly benefits from using expectation over scalarizations.}
    \label{fig:results-expectation-vs-random}
\end{figure}

\section{Conclusion}
\label{sec:discussion}

In this work we have presented an extension of the maKG acquisition function to the case of decoupled objectives. We have provided a theoretical guarantee for the convergence of the optimization and have experimentally demonstrated that it performs as well as other state-of-the-art acquisition functions such as HVKG and JES-LB. This is surprising given it relies on linear scalarizations - something which we use to keep the numerical approximation computationally feasible, and which simplifies the proof of convergence.

For future work, we are planning to compare the algorithm with a version which uses other utility functions such as (augmented) Chebyshev scalarizations to quantify any improvement in convergence obtained by using a Pareto compliant utility function.
It would also be interesting to extend this work to the case where the objective costs must be learned, and to combine our method with those which tackle variation of the cost over the input space \cite{snoek2012practical,astudillo2021multistep,lee2021nonmyopic}.


\begin{credits}
\subsubsection{\ackname} The first author was supported by the Engineering and Physical Sciences Research Council through the Mathematics of Systems II Centre for Doctoral Training at the University of Warwick (reference EP/S022244/1). The second author was supported by FWO (Belgium) grant number 1216021N and the Belgian Flanders AI Research Program (VCCM Core Lab).

\subsubsection{\discintname}
The authors have no competing interests to declare that are relevant to the content of this article.
\end{credits}

\bibliographystyle{splncs04}
\bibliography{references}

\begin{thebibliography}{10}
\providecommand{\url}[1]{\texttt{#1}}
\providecommand{\urlprefix}{URL }
\providecommand{\doi}[1]{https://doi.org/#1}

\bibitem{allmendinger2015nonuniform}
Allmendinger, R., Handl, J., Knowles, J.: Multiobjective optimization: When objectives exhibit non-uniform latencies. European Journal of Operational Research  \textbf{243},  497--513 (6 2015). \doi{10.1016/j.ejor.2014.09.033}

\bibitem{allmendinger2021hetero}
Allmendinger, R., Knowles, J.: Heterogeneous objectives: State-of-the-art and future research. arXiv arXiv:2103.15546 (2 2021)

\bibitem{alvarez2011multioutputgp}
\'Alvarez, M.A., Rosasco, L., Lawrence, N.D.: Kernels for vector-valued functions: A review. Foundations and Trends® in Machine Learning  \textbf{4}(3),  195--266 (2012). \doi{10.1561/2200000036}

\bibitem{astudillo2017makg}
Astudillo, R., Frazier, P.I.: Multi-attribute {Bayesian} optimization under utility uncertainty. In: NIPS Workshop on {B}ayesian Optimization. Long Beach, California, USA (December 2017), \url{https://bayesopt.github.io/papers/2017/41.pdf}

\bibitem{astudillo2021multistep}
Astudillo, R., Jiang, D., Balandat, M., Bakshy, E., Frazier, P.: Multi-step budgeted {B}ayesian optimization with unknown evaluation costs. Advances in Neural Information Processing Systems  \textbf{34},  20197--20209 (2021)

\bibitem{azais2009levelsets}
Aza{\"i}s, J.M., Wschebor, M.: Level sets and extrema of random processes and fields. John Wiley \& Sons (2009)

\bibitem{balandat2020botorch}
Balandat, M., Karrer, B., Jiang, D., Daulton, S., Letham, B., Wilson, A.G., Bakshy, E.: Botorch: A framework for efficient monte-carlo {Bayesian} optimization. In: Larochelle, H., Ranzato, M., Hadsell, R., Balcan, M., Lin, H. (eds.) Advances in Neural Information Processing Systems. vol.~33, pp. 21524--21538. Curran Associates, Inc. (2020), \url{https://proceedings.neurips.cc/paper/2020/file/f5b1b89d98b7286673128a5fb112cb9a-Paper.pdf}

\bibitem{bect2009sur}
Bect, J., Bachoc, F., Ginsbourger, D.: A supermartingale approach to gaussian process based sequential design of experiments. Bernoulli  \textbf{25} (11 2019). \doi{10.3150/18-BEJ1074}

\bibitem{blank2022hetero}
Blank, J., Deb, K.: Handling constrained multi-objective optimization problems with heterogeneous evaluation times: proof-of-principle results. Memetic Computing  \textbf{14},  135--150 (6 2022). \doi{10.1007/s12293-022-00362-z}

\bibitem{chugh2018hkrvea}
Chugh, T., Allmendinger, R., Ojalehto, V., Miettinen, K.: Surrogate-assisted evolutionary biobjective optimization for objectives with non-uniform latencies. In: Proceedings of the Genetic and Evolutionary Computation Conference. p. 609–616. GECCO '18, Association for Computing Machinery, New York, NY, USA (2018). \doi{10.1145/3205455.3205514}

\bibitem{couckuyt2022bayesian}
Couckuyt, I., Rojas~Gonzalez, S., Branke, J.: Bayesian optimization: tutorial. In: Proceedings of the Genetic and Evolutionary Computation Conference Companion. pp. 843--863 (2022)

\bibitem{daulton2023hvkg}
Daulton, S., Balandat, M., Bakshy, E.: Hypervolume knowledge gradient: A lookahead approach for multi-objective {B}ayesian optimization with partial information. In: Proceedings of the 40th International Conference on Machine Learning. Proceedings of Machine Learning Research, vol.~202, pp. 7167--7204. PMLR (23--29 Jul 2023), \url{https://proceedings.mlr.press/v202/daulton23a.html}

\bibitem{deb2002nsgaii}
Deb, K., Pratap, A., Agarwal, S., Meyarivan, T.: A fast and elitist multiobjective genetic algorithm: {NSGA-II}. IEEE Transactions on Evolutionary Computation  \textbf{6}(2),  182--197 (2002). \doi{10.1109/4235.996017}

\bibitem{frazier2018botutorial}
Frazier, P.I.: {B}ayesian optimization. In: Recent Advances in Optimization and Modeling of Contemporary Problems, pp. 255--278. INFORMS (10 2018). \doi{10.1287/educ.2018.0188}

\bibitem{frazier2008kg}
Frazier, P.I., Powell, W.B., Dayanik, S.: A knowledge-gradient policy for sequential information collection. SIAM Journal on Control and Optimization  \textbf{47}(5),  2410--2439 (2008). \doi{10.1137/070693424}

\bibitem{frazier2009kg}
Frazier, P.I., Powell, W.B., Dayanik, S.: The knowledge-gradient policy for correlated normal beliefs. INFORMS Journal on Computing  \textbf{21},  599--613 (11 2009). \doi{10.1287/ijoc.1080.0314}

\bibitem{garnett2023bayesopt}
Garnett, R.: {B}ayesian optimization. Cambridge University Press (2023), \url{https://bayesoptbook.com/}

\bibitem{hansen1994evaluating}
Hansen, M.P., Jaszkiewicz, A.: Evaluating the quality of approximations to the non-dominated set. Tech. rep., Institute of Mathematical Modelling, Technical University of Denmark (1994)

\bibitem{hernandezlobato2016pesmo}
Hernandez-Lobato, D., Hernandez-Lobato, J., Shah, A., Adams, R.: Predictive entropy search for multi-objective {B}ayesian optimization. In: Proceedings of The 33rd International Conference on Machine Learning. Proceedings of Machine Learning Research, vol.~48, pp. 1492--1501. PMLR, New York, New York, USA (20--22 Jun 2016), \url{https://proceedings.mlr.press/v48/hernandez-lobatoa16.html}

\bibitem{hernandezlobato2016pesc}
Hern\'{a}ndez-Lobato, J.M., Gelbart, M.A., Adams, R.P., Hoffman, M.W., Ghahramani, Z.: A general framework for constrained {B}ayesian optimization using information-based search. Journal of Machine Learning Research  \textbf{17}(160),  1--53 (2016), \url{http://jmlr.org/papers/v17/15-616.html}

\bibitem{knowles2006parego}
Knowles, J.: {ParEGO}: A hybrid algorithm with on-line landscape approximation for expensive multiobjective optimization problems. IEEE Transactions on Evolutionary Computation  \textbf{10},  50--66 (2 2006). \doi{10.1109/TEVC.2005.851274}

\bibitem{lacour2017boxdecomp}
Lacour, R., Klamroth, K., Fonseca, C.M.: A box decomposition algorithm to compute the hypervolume indicator. Computers \& Operations Research  \textbf{79},  347--360 (2017). \doi{https://doi.org/10.1016/j.cor.2016.06.021}, \url{https://www.sciencedirect.com/science/article/pii/S0305054816301538}

\bibitem{lee2021nonmyopic}
Lee, E.H., Eriksson, D., Perrone, V., Seeger, M.: A nonmyopic approach to cost-constrained {Bayesian} optimization. In: de~Campos, C., Maathuis, M.H. (eds.) Proceedings of the Thirty-Seventh Conference on Uncertainty in Artificial Intelligence. Proceedings of Machine Learning Research, vol.~161, pp. 568--577. PMLR (27--30 Jul 2021)

\bibitem{loka2022chvei}
Loka, N., Couckuyt, I., Garbuglia, F., Spina, D., Nieuwenhuyse, I.V., Dhaene, T.: Bi-objective {B}ayesian optimization of engineering problems with cheap and expensive cost functions. Engineering with Computers  (1 2022). \doi{10.1007/s00366-021-01573-7}

\bibitem{mamun2022hetero}
Mamun, M.M., Singh, H.K., Ray, T.: An approach for computationally expensive multi-objective optimization problems with independently evaluable objectives. Swarm and Evolutionary Computation  \textbf{75},  101146 (12 2022). \doi{10.1016/j.swevo.2022.101146}

\bibitem{miettinen1999moo}
Miettinen, K.: Nonlinear multiobjective optimization, vol.~12. Springer Science \& Business Media (1999)

\bibitem{paria2020scalarizations}
Paria, B., Kandasamy, K., P{\'{o}}czos, B.: A flexible framework for multi-objective {B}ayesian optimization using random scalarizations. In: Adams, R.P., Gogate, V. (eds.) Proceedings of The 35th Uncertainty in Artificial Intelligence Conference. Proceedings of Machine Learning Research, vol.~115, pp. 766--776. PMLR (22--25 Jul 2020), \url{https://proceedings.mlr.press/v115/paria20a.html}

\bibitem{pearce2017robust}
Pearce, M., Branke, J.: {B}ayesian simulation optimization with input uncertainty. In: 2017 Winter Simulation Conference (WSC). pp. 2268--2278. IEEE (12 2017). \doi{10.1109/WSC.2017.8247958}

\bibitem{pearce2018multitask}
Pearce, M., Branke, J.: Continuous multi-task {B}ayesian optimisation with correlation. European Journal of Operational Research  \textbf{270},  1074--1085 (11 2018). \doi{10.1016/j.ejor.2018.03.017}

\bibitem{pisier2016martingalesbanach}
Pisier, G.: Martingales in Banach Spaces. Cambridge Studies in Advanced Mathematics, Cambridge University Press (2016). \doi{10.1017/CBO9781316480588}

\bibitem{rasmussenwilliams2006gps}
Rasmussen, C.E., Williams, C.K.I.: Gaussian Processes for Machine Learning (Adaptive Computation and Machine Learning). The MIT Press (2006). \doi{10.5555/1162254}

\bibitem{shahriari2016bayesopt}
Shahriari, B., Swersky, K., Wang, Z., Adams, R.P., de~Freitas, N.: Taking the human out of the loop: A review of {B}ayesian optimization. Proceedings of the IEEE  \textbf{104}(1),  148--175 (2016). \doi{10.1109/JPROC.2015.2494218}

\bibitem{snoek2012practical}
Snoek, J., Larochelle, H., Adams, R.P.: Practical {B}ayesian optimization of machine learning algorithms. In: Pereira, F., Burges, C., Bottou, L., Weinberger, K. (eds.) Advances in Neural Information Processing Systems. vol.~25. Curran Associates, Inc. (2012)

\bibitem{sobol1967sobolseq}
Sobol', I.M.: On the distribution of points in a cube and the approximate evaluation of integrals. Zhurnal Vychislitel'noi Matematiki i Matematicheskoi Fiziki  \textbf{7}(4),  784--802 (1967)

\bibitem{suzuki2020pfes}
Suzuki, S., Takeno, S., Tamura, T., Shitara, K., Karasuyama, M.: Multi-objective {B}ayesian optimization using {P}areto-frontier entropy. In: III, H.D., Singh, A. (eds.) Proceedings of the 37th International Conference on Machine Learning. Proceedings of Machine Learning Research, vol.~119, pp. 9279--9288. PMLR (13--18 Jul 2020), \url{https://proceedings.mlr.press/v119/suzuki20a.html}

\bibitem{tu2022jes}
Tu, B., Gandy, A., Kantas, N., Shafei, B.: Joint entropy search for multi-objective {B}ayesian optimization. In: Koyejo, S., Mohamed, S., Agarwal, A., Belgrave, D., Cho, K., Oh, A. (eds.) Advances in Neural Information Processing Systems. vol.~35, pp. 9922--9938. Curran Associates, Inc. (2022), \url{https://proceedings.neurips.cc/paper_files/paper/2022/file/4086fe59dc3584708468fba0e459f6a7-Paper-Conference.pdf}

\bibitem{tu2023r2optim}
Tu, B., Kantas, N., Lee, R.M., Shafei, B.: Multi-objective optimisation via the {R2} utilities. arXiv  (2024). \doi{10.48550/arXiv.2305.11774}

\bibitem{braembussche2008turbomachinery}
{Van den Braembussche}, R.A.: Numerical optimization for advanced turbomachinery design. In: Th\'evenin, D., Janiga, G. (eds.) Optimization and Computational Fluid Dynamics, chap.~6, pp. 147--189. Springer (2008). \doi{10.1007/978-3-540-72153-6}

\bibitem{wang2022hetero}
Wang, X., Jin, Y., Schmitt, S., Olhofer, M.: Alleviating search bias in {B}ayesian evolutionary optimization with many heterogeneous objectives. arXiv arxiv:2208.12217 (2022). \doi{10.48550/ARXIV.2208.12217}

\bibitem{wu2016paralleldiscretekg}
Wu, J., Frazier, P.: The parallel knowledge gradient method for batch {Bayesian} optimization. In: Lee, D., Sugiyama, M., Luxburg, U., Guyon, I., Garnett, R. (eds.) Advances in Neural Information Processing Systems. vol.~29. Curran Associates, Inc. (2016), \url{https://proceedings.neurips.cc/paper/2016/file/18d10dc6e666eab6de9215ae5b3d54df-Paper.pdf}

\end{thebibliography}

\newpage
\appendix
\newcounter{tmpcounter}

\section{Theoretical Results}
\label{sec:proofs}
In this appendix we prove theoretical results associated with the cost-weighted multi-objective knowledge gradient (C-MOKG). The appendix culminates with a proof that, for any scalarization weights, \(\vect{\lambda}\), the estimator of the optimum is asymptotically consistent.
The structure of the proof mostly follows that of Bect et al. \cite{bect2009sur}, without as much of the complex machinery. Unfortunately, we cannot directly apply the results from \cite{bect2009sur} because they do not cover multi-objective GP models or the case of random scalarizations.

Since conditional expectations are only defined up to a null event, we qualify the results in this section with `almost surely' (a.s.). The reader should interpret most equalities and inequalities in the proofs of these results to hold almost surely.

\subsection{Statistical Model}
For the benefit of the reader, we recall the statistical model used.
Let \(\vect{f} \sim \mathcal{GP}(\vect{\mu}, \matr{K})\) be a multi-output Gaussian process with compact index set \(\mathcal{X} \subset \mathbb{R}^D\) and defined on a probability space \((\Omega, \mathcal{F}, \mathbb{P})\). here \(\vect{\mu} : \mathcal{X} \to \mathbb{R}^M\) is the mean function and \(\matr{K}: \mathcal{X} \times \mathcal{X} \to \mathbb{R}^{M \times M}\) is the covariance function. We assume that \(\vect{\mu}\) and \(\matr{K}\) are continuous, and further we use the version of \(\vect{f}\) with continuous sample paths.
Importantly, we do not make a distinction between the GP \(\vect{f}\) and the `true' function \(\vect{f}^*\).

Assume that \(\vect{f}\) can be observed according to a model
\begin{equation} \label{eq:app-obs-model}
    y_m = f_m(\vect{x}) + \varepsilon_m
\end{equation}
where \(m \in [M] = \{1, \dots, M\}\) indexes the component of \(\vect{f}\) being evaluated, \(\vect{x} \in \mathcal{X}\) is the input location and \(\varepsilon_m \sim \mathcal{N}(0, \sigma_m^2)\) denotes the observation noise.
Let \(\vect{c} \in (0, \infty)^M\) be a vector of strictly positive evaluation costs associated with each objective.
We consider sequentially selected (data-dependent) design points and indices \(\vect{x}_1, m_1, \vect{x}_2, m_2, \dots\) and denote the corresponding observations by \(y_{1, m_1}, y_{2, m_2}, \dots\). We write \(\varepsilon_{1,m_1}, \varepsilon_{2, m_2}, \dots\) for the independent noise terms added for each observation.

Let \(\Lambda = \{\vect{\lambda} \in [0, \infty)^M : \sum_{j=1}^M \lambda_j = 1\}\) denote the standard simplex in \(\mathbb{R}^M\).
We will prove results for both the case of expectation over scalarizations and of random scalarizations.
In both cases, we will assume that the probability density \(p(\vect{\lambda})\) is strictly positive everywhere on \(\Lambda\).
For the algorithm which uses random scalarizations, we denote the sequence of scalarization weights by \(\vect{\lambda}_1, \vect{\lambda}_2, \dots\).
For each \(n \in \mathbb{N}_0\), denote by \(\mathcal{F}_n = \sigma(\{\vect{x}_j, m_j, y_{j, m_j}, \varepsilon_{j, m_j}\}_{j=1}^n \cup \{\vect{\lambda}_j\}_{j=1}^\infty)\), the \(\sigma\)-algebra generated by all the information available at time \(n\) along with all the full sequence of scalarization weights. Of course, while harmless, we only need to include the \(\vect{\lambda}_j\) in the case of random scalarizations.
In a slight change from the notation in the main text, we then write \(\mathbb{E}_n[\,\cdot\,]\) for the expectation conditional on \(\mathcal{F}_n\). In order to cleanly handle both algorithms, when a random scalarization weight vector, \(\vect{\lambda}\), is present, we shall consider it part of the conditioned variables,
\[\mathbb{E}_n[\,\cdot\,] = \mathbb{E}[\,\cdot \,|\, \sigma(\mathcal{F}_n, \vect{\lambda})].\]
Then, in order to denote expectation over \(\vect{\lambda}\) alone, we use \(\mathbb{E}_{\vect{\lambda} \sim p(\vect{\lambda})}[\,\cdot\,]\).

Recall our two C-MOKG acquisition functions from the main text. While our definition of \(\mathbb{E}_n\) now includes the full sequence of random scalarizations instead of just those up to time \(n\), thanks to the causal structure, this does not change the formulae,
\begin{subequations} \label{eq:app-cmokg}
\begin{align}
    \alpha_{\text{C-MOKG}}^n(\vect{x}, m; \vect{\lambda}, \vect{c}) &= \frac{1}{c_m} \left( \mathbb{E}_n\!\left[ \max_{\vect{x}' \in \mathcal{X}} \mu_{n+}^s(\vect{x}'; \vect{x}, m, \vect{\lambda}) \right] - \max_{\vect{x}' \in \mathcal{X}} \mu_n^s(\vect{x}'; \vect{\lambda}) \right), \label{eq:app-cmokg-r} \\
    \overline{\alpha}_{\text{C-MOKG}}^n(\vect{x}, m; \vect{c}) &= \mathbb{E}_{\vect{\lambda} \sim p(\vect{\lambda})}\left[ \alpha_{\text{C-MOKG}}^n(\vect{x}, m; \vect{\lambda}, \vect{c}) \right], \label{eq:app-cmokg-e}
\end{align}
\end{subequations}
where \(\mu_n^s(\vect{x}'; \vect{\lambda}) = \mathbb{E}_n[\vect{\lambda} \cdot \vect{f}(\vect{x})]\), \(\mu_{n+}^s(\vect{x}'; \vect{x}, m, \vect{\lambda}) = \mathbb{E}_n[\vect{\lambda} \cdot \vect{f}(\vect{x}') \,|\, y_m^{\vect{x}} ]\) and \(y_m^{\vect{x}} = f_m(\vect{x}) + \varepsilon_m\) is the hypothesized next observation. Note the superscript \(s\) used in this appendix, to distinguish the posterior mean of the scalarized GP from the posterior mean of the unscalarized GP which will be introduced before \cref{thm:app-unif-convergence-cond-gps}.

Informally, we will assume that the design points \(\vect{x}_1, \vect{x}_2, \dots\) and objective indices \(m_1, m_2, \dots\) are chosen to maximize the relevant acquisition function at each step. However, the knowledge gradient is not continuous everywhere so it is not obvious that it attains its maximum. Further, in practice we will never perfectly maximize the acquisition function. For these reasons, we assume that the design points are chosen only to be approximate maximizers of the acquisition functions. That is, we assume that there exists a sequence \(\eta = (\eta_n)_{n=0}^\infty\) of small, non-negative real numbers satisfying \(\eta_n \to 0\) as \(n \to \infty\) and such that one of
\begin{subequations}\label{eq:app-quasi-sur}
\begin{align}
    \begin{split}
        \forall n \in \mathbb{N}_0\quad &\alpha_{\text{C-MOKG}}^n(\vect{x}_{n+1}, m_{n+1}\,;\; \vect{\lambda}_{n+1}, \vect{c})\\
        &\qquad > \sup_{\substack{\vect{x} \in \mathcal{X} \\ m \in [M]}} \alpha_{\text{C-MOKG}}^n(\vect{x}, m\,;\; \vect{\lambda}_{n+1}, \vect{c}) - \eta_n, 
    \end{split}
    \label{eq:app-quasi-sur-1}
    \\
    \begin{split}
        \forall n \in \mathbb{N}_0\quad
        &\overline{\alpha}_{\text{C-MOKG}}^n(\vect{x}_{n+1}, m_{n+1}; \vect{c})\\
        &\qquad > \sup_{\substack{\vect{x} \in \mathcal{X} \\ m \in [M]}} \overline{\alpha}_{\text{C-MOKG}}^n(\vect{x}, m; \vect{c}) - \eta_n,
    \end{split}
    \label{eq:app-quasi-sur-2}
\end{align}
\end{subequations}
depending on whether we are using random scalarizations or expectation over scalarization weights.
This is referred to as an \emph{\(\eta\)-quasi-SUR} sequential design by Bect et al. \cite{bect2009sur}\footnote{Note that in \cite{bect2009sur}, the character \(\varepsilon\) is used in place of \(\eta\)}.

\subsection{\texorpdfstring{Convergence of \(\alpha_{\text{C-MOKG}}^n\) to Zero}{Convergence of C-MOKG to zero}}
\begin{definition}\label{defn:app-resid-uncert}
    By analogy with \cite{bect2009sur}, let us set up some notation for the \emph{residual uncertainty} associated with each acquisition function at step \(n \in \mathbb{N}_0\),
    \begin{subequations}
    \begin{align}
        \forall \vect{\lambda} \in \Lambda\quad H_n(\vect{\lambda}) &= \mathbb{E}_n\left[\max_{\vect{x}' \in \mathcal{X}} \vect{\lambda} \cdot \vect{f}(\vect{x}') \right] - \max_{\vect{x}' \in \mathcal{X}} \mathbb{E}_n[ \vect{\lambda} \cdot \vect{f}(\vect{x}')], \\
        \overline{H}_n &= \mathbb{E}_{\vect{\lambda} \sim p(\vect{\lambda})}[H_n(\vect{\lambda})]
    \end{align}
    \end{subequations}
\end{definition}

\begin{lemma}
    The residual uncertainty in \cref{defn:app-resid-uncert} is well defined. That is,
    \[
        \forall n \in \mathbb{N}_0\, \forall \vect{\lambda} \in \Lambda\quad H_n(\vect{\lambda}) < \infty \quad\text{a.s.}
        \qquad\text{and}\qquad
        \forall n \in \mathbb{N}_0 \quad \overline{H}_n < \infty \quad\text{a.s..}
    \]
\end{lemma}
\begin{proof}
    We will consider the first form first. Let \(\vect{\lambda} \in \Lambda\).
    Then, since all components of \(\vect{\lambda}\) lie between 0 and 1, we have \(|\vect{\lambda} \cdot \vect{f}(\vect{x}')| \leq \| \vect{f}(\vect{x}') \|_1\) for all \(\vect{x}' \in \mathcal{X}\).
    Therefore,
    \begin{align*}
        H_n(\vect{\lambda}) &\leq \mathbb{E}_n\left[ \max_{\vect{x}' \in \mathcal{X}} | \vect{\lambda} \cdot \vect{f}(\vect{x}') | \right] + \max_{\vect{x}' \in \mathcal{X}} \mathbb{E}_n |\vect{\lambda} \cdot \vect{f}(\vect{x}') | \\
        &\leq 2 \mathbb{E}_n\left[ \max_{\vect{x}' \in \mathcal{X}} | \vect{\lambda} \cdot \vect{f}(\vect{x}') | \right] \\
        &\leq 2 \mathbb{E}_n\left[\max_{\vect{x}' \in \mathcal{X}} \|\vect{f}(\vect{x}')\|_1 \right] \\
        &\leq 2 \sum_{m=1}^M \mathbb{E}_n\left[\max_{\vect{x}' \in \mathcal{X}} |f_m(\vect{x}')| \right]
        < \infty.
    \end{align*}
    The final inequality here follows since each \(f_m\) has continuous sample paths and \(\mathcal{X}\) is compact. For example, use Theorem 2.9 from \cite{azais2009levelsets} and note that for any \(m\) and any \(\vect{x} \in \mathcal{X}\), we have \(\mathbb{E}_n[\max_{\vect{x}' \in \mathcal{X}} |f_m(\vect{x}')|] \leq 2 \mathbb{E}_n[\max_{\vect{x}' \in \mathcal{X}} f_m(\vect{x}')] + \mathbb{E}_n|f_m(\vect{x})|\).
    
    To show the same for \(\overline{H}_n\), we simply take expectation over \(\vect{\lambda} \sim p(\Lambda)\). That is,
    \[\overline{H}_n = \mathbb{E}_{\vect{\lambda} \sim p(\vect{\lambda})}[H_n(\vect{\lambda})] \leq \sum_{m=1}^M \mathbb{E}_n\left[\sup_{\vect{x} \in \mathcal{X}} |f_m(\vect{x})| \right] < \infty.\]
\end{proof}

\begin{remark}
    Consider the case of random scalarizations. Since \(\vect{x}_{n+1}\) and \(m_{n+1}\) are deterministic after conditioning on \(\mathcal{F}_n\),%
    \footnote{Formally, we say that \(\vect{x}_{n+1}\) and \(m_{n+1}\) are \(\mathcal{F}_n\)-measurable.}
    we can substitute them directly inside the outer expectation in \(\alpha_{\text{C-MOKG}}^n(\vect{x}_{n+1}, m_{n+1};\, \vect{\lambda}_{n+1}, \vect{c})\).
    Further, for all \(\vect{x}' \in \mathcal{X}\), \(\mu_{n+}^s(\vect{x}'; \vect{x}_{n+1}, m_{n+1}, \vect{\lambda}_{n+1}) = \mu_{n+1}^s(\vect{x}'; \vect{\lambda}_{n+1})\) and we have
    \begin{align}
        &\mkern-36mu \alpha_{\text{C-MOKG}}^n(\vect{x}_{n+1}, m_{n+1};\, \vect{\lambda}_{n+1}, \vect{c}) \nonumber \\
        &= \frac{1}{c_{m_{n+1}}} \left( \mathbb{E}_n \!\left[ \max_{\vect{x}' \in \mathcal{X}} \mu_{n+1}^s(\vect{x}'; \vect{\lambda}_{n+1}) \right] - \max_{\vect{x}' \in \mathcal{X}} \mu_n^s(\vect{x}'; \vect{\lambda}_{n+1}) \right) \nonumber \\
        &= \frac{1}{c_{m_{n+1}}} \left( \mathbb{E}_n \!\left[ \max_{\vect{x}' \in \mathcal{X}} \mathbb{E}_{n+1}\left[\vect{\lambda}_{n+1} \cdot \vect{f}(\vect{x}') \right] \right] - \max_{\vect{x}' \in \mathcal{X}} \mathbb{E}_n[\vect{\lambda}_{n+1} \cdot \vect{f}(\vect{x}')] \right) \nonumber \\
        &= \frac{1}{c_{m_{n+1}}} \Bigl( H_n(\vect{\lambda}_{n+1}) - \mathbb{E}_n[H_{n+1}(\vect{\lambda}_{n+1})] \Bigr) \quad\text{a.s.} \label{eq:app-anxn-1}
    \end{align}
    We obtain a similar result for the case of average scalarizations, giving
    \begin{align}\label{eq:app-anxn-2}
        \overline{\alpha}_{\text{C-MOKG}}^n(\vect{x}_{n+1}, m_{n+1}; \vect{c})
        &= \mathbb{E}_{\vect{\lambda} \sim p(\vect{\lambda})}\bigl[ \alpha_{\text{C-MOKG}}^n(\vect{x}_{n+1}, m_{n+1}, \vect{\lambda}; \vect{c}) \bigr] \nonumber \\
        &= \frac{1}{c_{m_{n+1}}} \Bigl(\overline{H}_n - \mathbb{E}_{n}\bigl[\,\overline{H}_{n+1}\bigr] \Bigr)\quad\text{a.s.}
    \end{align}
\end{remark}

We begin with three closely related lemmas.
\setcounter{tmpcounter}{\value{theorem}}
\setcounter{theorem}{0}
\begin{lemma}[Restated from main text] \label{thm:app-cmokg-nonneg}
    Both forms of the C-MOKG are non-negative. That is, for all \(n \in \mathbb{N}_0\), \(\vect{x} \in \mathcal{X}\), \(m \in \{1, \dots, M\}\) and all \(\vect{\lambda} \in \Lambda\),
    \begin{equation*}
        \alpha_{\text{C-MOKG}}^n(\vect{x}, m;\, \vect{\lambda}, \vect{c}) \geq 0 \quad\text{a.s.}
        \qquad\text{and}\qquad
        \overline{\alpha}_{\text{C-MOKG}}^n(\vect{x}, m; \vect{c}) \geq 0 \quad\text{a.s.}.
    \end{equation*}
\end{lemma}
\setcounter{theorem}{\value{tmpcounter}}
\begin{proof}
    We first prove the result for \(\alpha_{\text{C-MOKG}}^n(\,\cdot\,, \cdot\,;\, \vect{\lambda}, \vect{c})\), which takes the scalarization weights \(\vect{\lambda}\) as an argument.
    Let \(n \in \mathbb{N}_0\), \(\vect{x} \in \mathcal{X}\), \(m \in \{1, \dots, M\}\) and let \(\vect{\lambda} \in \Lambda\) be random.
    Then, for any \(\vect{x}'' \in \mathcal{X}\),
    \[
        \max_{\vect{x}' \in \mathcal{X}} \mu_{n+}^s(\vect{x}'; \vect{x}, m, \vect{\lambda})
        \geq
        \mu_{n+}^s(\vect{x}''; \vect{x}, m, \vect{\lambda}).
    \]
    Taking expectation conditional on \(\mathcal{F}_n\) and \(\vect{\lambda}\) using \(\mathbb{E}_n\) gives
    \[
        \mathbb{E}_n\!\left[ \max_{\vect{x}' \in \mathcal{X}} \mu_{n+}^s(\vect{x}'; \vect{x}, m, \vect{\lambda}) \right]
        \geq
        \mathbb{E}_n\!\bigl[ \mu_{n+}^s(\vect{x}''; \vect{x}, m, \vect{\lambda}) \bigr]
        = \mu_n^s(\vect{x}''; \vect{\lambda}).
    \]
    Finally, this holds for all \(\vect{x}'' \in \mathcal{X}\) so certainly holds for the maximum
    \begin{align*}
        & \mathbb{E}_n\!\left[ \max_{\vect{x}' \in \mathcal{X}} \mu_{n+}^s(\vect{x}'; \vect{x}, m, \vect{\lambda}) \right]
        \geq
        \max_{\vect{x}' \in \mathcal{X}} \mu_n^s(\vect{x}'; \vect{\lambda})
        \\
        \Rightarrow\quad& \alpha_{\text{C-MOKG}}^n(\vect{x}, m;\, \vect{\lambda}, \vect{c}) \\
        &\qquad\qquad = \frac{1}{c_m} \left( \mathbb{E}_n\!\left[ \max_{\vect{x}' \in \mathcal{X}} \mu_{n+}^s(\vect{x}'; \vect{x}, m, \vect{\lambda}) \right]
        - \max_{\vect{x}' \in \mathcal{X}} \mu_n^s(\vect{x}'; \vect{\lambda}) \right)
        \geq 0.
    \end{align*}
    To show the result for \(\overline{\alpha}_{\text{C-MOKG}}\), we simply take expectation to integrate out \(\vect{\lambda}\),
    \[\overline{\alpha}_{\text{C-MOKG}}^n(\vect{x}, m; \vect{c}) = \mathbb{E}_{\vect{\lambda} \sim p(\vect{\lambda})}[\alpha_{\text{C-MOKG}}^n(\vect{x}, m;\, \vect{\lambda}, \vect{c})] \geq 0.\]
\end{proof}

\begin{lemma}\label{thm:app-hn-bigger-than-cmokg}
    In both cases, the residual uncertainty is at least the C-MOKG scaled by the objective cost. That is, for all \(n \in \mathbb{N}_0\), \(\vect{x} \in \mathcal{X}\), \(m \in \{1, \dots, M\}\) and all \(\vect{\lambda} \in \Lambda\),
    \begin{equation*}
        c_m \alpha_{\text{C-MOKG}}^n(\vect{x}, m;\, \vect{\lambda}, \vect{c}) \leq H_n(\vect{\lambda})
        \qquad\text{and}\qquad
        c_m \overline{\alpha}_{\text{C-MOKG}}^n(\vect{x}, m; \vect{c}) \leq \overline{H}_n.
    \end{equation*}
\end{lemma}
\begin{proof}
    Let \(\vect{\lambda} \in \Lambda\). By a very similar argument as was used to show that the C-MOKG was non-negative in \cref{thm:app-cmokg-nonneg},
    \begin{alignat*}{3}
        && \forall \vect{x}'' \in \mathcal{X}& \quad& \max_{\vect{x}' \in \mathcal{X}} \vect{\lambda} \cdot \vect{f}(\vect{x}') &\geq \vect{\lambda} \cdot \vect{f}(\vect{x}'') \\
        &\Rightarrow\quad & \forall n \in \mathbb{N}_0\; \forall \vect{x}'' \in \mathcal{X}& \quad& \mathbb{E}_n\left[\max_{\vect{x}' \in \mathcal{X}} \vect{\lambda} \cdot \vect{f}(\vect{x}') \,\middle|\, y_m^{\vect{x}} \right] &\geq \mathbb{E}_n\left[\vect{\lambda} \cdot \vect{f}(\vect{x}'') \,\middle|\, y_m^{\vect{x}} \right] \\
        &&&&&= \mu_{n+}^s(\vect{x}''; \vect{x}, m, \vect{\lambda}) \\
        &\Rightarrow\quad & \forall n \in \mathbb{N}_0&\quad & \mathbb{E}_n\left[\max_{\vect{x}' \in \mathcal{X}} \vect{\lambda} \cdot \vect{f}(\vect{x}') \,\middle|\, y_m^{\vect{x}} \right] &\geq \max_{\vect{x}' \in \mathcal{X}} \mu_{n+}^s(\vect{x}'; \vect{x}, m, \vect{\lambda}) \\
        &\Rightarrow\quad & \forall n \in \mathbb{N}_0&\quad & \mathbb{E}_n\!\left[\max_{\vect{x}' \in \mathcal{X}} \vect{\lambda} \cdot \vect{f}(\vect{x}')\right] &\geq \mathbb{E}_n\!\left[\max_{\vect{x}' \in \mathcal{X}} \mu_{n+}^s(\vect{x}'; \vect{x}, m, \vect{\lambda})\right] \\
        &\Rightarrow\quad & \forall n \in \mathbb{N}_0&\quad & H_n(\vect{\lambda}) &\geq c_m \alpha_{\text{C-MOKG}}^n(\vect{x}, m;\, \vect{\lambda}, \vect{c}).
    \end{alignat*}
    To establish that \(\overline{H}_n \geq c_m \overline{\alpha}_{\text{C-MOKG}}^n(\vect{x}, m; \vect{c})\), we simply take expectation over \(\vect{\lambda} \sim p(\Lambda)\).
\end{proof}

\begin{lemma}\label{thm:app-Hn-supmart}
    For all \(\vect{\lambda} \in \Lambda\), the sequence \((H_n(\vect{\lambda}))_{n \in \mathbb{N}_0}\) is a non-negative supermartingale with respect to the filtration \((\mathcal{F}_n)_{n=0}^\infty\).
    Similarly, \((\overline{H}_n)_{n \in \mathbb{N}_0}\) is a non-negative supermartingale with respect to the same filtration.
    That is,
    \begin{alignat*}{4}
        \forall \vect{\lambda} \in \Lambda\; \forall n \in \mathbb{N}_0&\quad & H_n(\vect{\lambda}) &\geq 0 \text{ a.s.} &
        \qquad&\text{and}\qquad &
        H_n(\vect{\lambda}) &\geq \mathbb{E}_n[H_{n+1}(\vect{\lambda})] \text{ a.s.} \\
        \forall n \in \mathbb{N}_0&\quad & \overline{H}_n &\geq 0 \text{ a.s.} &
        \qquad&\text{and}\qquad &
        \overline{H}_n &\geq \mathbb{E}_n[\,\overline{H}_{n+1}] \text{ a.s.}
    \end{alignat*}
\end{lemma}
\begin{proof}
    Let \(\vect{\lambda} \in \Lambda\).
    By \cref{thm:app-cmokg-nonneg,thm:app-hn-bigger-than-cmokg}, for all \(\vect{x} \in \mathcal{X}\) and all \(m \in \{1, \dots, M\}\),
    \[
        H_n(\vect{\lambda}) \geq c_m \alpha_{\text{C-MOKG}}^n(\vect{x}, m;\, \vect{\lambda}, \vect{c}) \geq 0
        \quad\text{and}\quad
        \overline{H}_n \geq c_m \overline{\alpha}_{\text{C-MOKG}}^n(\vect{x}, m; \vect{c}) \geq 0
    \]
    almost surely.
    Further, by a similar argument used to prove \cref{thm:app-cmokg-nonneg,thm:app-hn-bigger-than-cmokg}, for all \(n \in \mathbb{N}_0\),
    \begin{alignat*}{3}
        && \forall \vect{x}'' \in \mathcal{X}&\quad & \max_{\vect{x}'} \mathbb{E}_{n+1}[\vect{\lambda} \cdot \vect{f}(\vect{x}')] &\geq \mathbb{E}_{n+1}[\vect{\lambda} \cdot \vect{f}(\vect{x}'')] \\
        &\Rightarrow\quad & \forall \vect{x}'' \in \mathcal{X}&\quad & \mathbb{E}_n \left[ \max_{\vect{x}' \in \mathcal{X}} \mathbb{E}_{n+1}[\vect{\lambda} \cdot \vect{f}(\vect{x}')]\right] &\geq \mathbb{E}_n[\vect{\lambda} \cdot \vect{f}(\vect{x}'')] \\
        &\Rightarrow\quad &&& \mathbb{E}_n \left[ \max_{\vect{x}' \in \mathcal{X}} \mathbb{E}_{n+1}[\vect{\lambda} \cdot \vect{f}(\vect{x}')]\right] &\geq \max_{\vect{x}' \in \mathcal{X}} \mathbb{E}_n[\vect{\lambda} \cdot \vect{f}(\vect{x}')]\\
        &\Rightarrow\quad &&& \mathbb{E}_n[H_{n+1}(\vect{\lambda})] &\leq H_n(\vect{\lambda}).
    \end{alignat*}
    
    Taking expectation over \(\vect{\lambda} \sim p(\vect{\lambda})\) gives the result for \(\overline{H}_n\).
\end{proof}

Next, we establish a Lipschitz property of the C-MOKG and residual uncertainty.
\begin{lemma} \label{thm:app-cmokg-and-h-lipschitz}
    The C-MOKG and residual uncertainty exhibit the following Lipschitz style properties in \(\vect{\lambda} \in \Lambda\), which hold almost surely. Let \(\vect{\lambda}, \vect{\lambda}' \in \Lambda\) and \(n \in \mathbb{N}_0\). Then
    \begin{enumerate}
        \item for all \(\vect{x} \in \mathcal{X}\) and \(m \in \{1, \dots, M\}\),
        \begin{multline*}
            \Bigl|\alpha_{\text{C-MOKG}}^n(\vect{x}, m; \vect{\lambda}, \vect{c}) - \alpha_{\text{C-MOKG}}^n(\vect{x}, m; \vect{\lambda}', \vect{c})\Bigr| \\
            \leq
            \frac{2}{c_m} \|\vect{\lambda} - \vect{\lambda}'\|_2 \; \mathbb{E}_n\left[ \max_{\vect{x}' \in \mathcal{X}} \|\vect{f}(\vect{x}')\|_2 \right];
        \end{multline*}
        \item \begin{equation*}
            |H_n(\vect{\lambda}) - H_n(\vect{\lambda}')|
            \leq
            2 \|\vect{\lambda} - \vect{\lambda}'\|_2 \; \mathbb{E}_n\left[ \max_{\vect{x}' \in \mathcal{X}} \|\vect{f}(\vect{x}')\|_2 \right].
        \end{equation*}
    \end{enumerate}
\end{lemma}
\begin{proof}
    We will begin with the C-MOKG. Let \(\vect{\lambda}, \vect{\lambda}' \in \Lambda\). Then for all \(\vect{x}, \vect{x}' \in \mathcal{X}\), \(m \in \{1, \dots, M\}\) and \(n \in \mathbb{N}_0\),
    \begin{align*}
        \mu_{n+}^s(\vect{x}'; \vect{x}, m, \vect{\lambda}) - \mu_{n+}^s(\vect{x}'; \vect{x}, m, \vect{\lambda}')
        &= \mathbb{E}_n\left[ (\vect{\lambda} - \vect{\lambda}') \cdot \vect{f}(\vect{x}') \,\middle|\, y_m^{\vect{x}} \right] \\
        &\leq \mathbb{E}_n\left[ \|\vect{\lambda} - \vect{\lambda}'\|_2 \; \| \vect{f}(\vect{x}') \|_2  \,\Big|\, y_m^{\vect{x}} \right] \\
        &\leq \|\vect{\lambda} - \vect{\lambda}'\|_2 \; \mathbb{E}_n\Bigl[\;\max_{\vect{x}'' \in \mathcal{X}} \|\vect{f}(\vect{x}'')\|_2 \,\Big|\, y_m^{\vect{x}}\Bigr]
    \end{align*}
    where the second line follows from the Cauchy-Schwarz inequality.
    We can do similarly for \(\mu_n^s(\vect{x}'; \vect{\lambda})\), giving
    \begin{equation*}
        \mu_n^s(\vect{x}'; \vect{\lambda}') - \mu_n^s(\vect{x}'; \vect{\lambda})
        \leq \|\vect{\lambda} - \vect{\lambda}'\|_2 \; \mathbb{E}_n\left[ \max_{\vect{x}'' \in \mathcal{X}} \|\vect{f}(\vect{x}'')\|_2 \right].
    \end{equation*}
    Hence,
    \begin{align*}
        c_m \Bigl( \alpha_{\text{C-MOKG}}^n(\vect{x}, m;\, &{}\vect{\lambda}, \vect{c}) - \alpha_{\text{C-MOKG}}^n(\vect{x}, m; \vect{\lambda}', \vect{c}) \Bigr) \\
        &= \mathbb{E}_n\left[ \max_{\vect{x}' \in \mathcal{X}} \mu_{n+}^s(\vect{x}'; \vect{x}, m, \vect{\lambda}) \right] - \max_{\vect{x}' \in \mathcal{X}} \mu_n^s(\vect{x}'; \vect{\lambda}) \\
        &\quad - \mathbb{E}_n\left[ \max_{\vect{x}' \in \mathcal{X}} \mu_{n+}^s(\vect{x}'; \vect{x}, m, \vect{\lambda}') \right] + \max_{\vect{x}' \in \mathcal{X}} \mu_n^s(\vect{x}'; \vect{\lambda}') \\
        &\leq \mathbb{E}_n\left[ \max_{\vect{x}' \in \mathcal{X}} 
\mu_{n+}^s(\vect{x}'; \vect{x}, m, \vect{\lambda}) - \mu_{n+}^s(\vect{x}'; \vect{x}, m, \vect{\lambda}') \right] \\
        &\quad + \max_{\vect{x}' \in \mathcal{X}} \Bigl( \mu_n^s(\vect{x}'; \vect{\lambda}') - \mu_n^s(\vect{x}'; \vect{\lambda}) \Bigr) \\
        &\leq \mathbb{E}_n\left[ \|\vect{\lambda} - \vect{\lambda}'\|_2 \; \mathbb{E}_n\Bigl[\;\max_{\vect{x}' \in \mathcal{X}} \|\vect{f}(\vect{x}')\|_2 \,\Big|\, y_m^{\vect{x}}\Bigr] \right] \\
        &\quad + \|\vect{\lambda} - \vect{\lambda}'\|_2 \; \mathbb{E}_n\Bigl[ \;\max_{\vect{x}' \in \mathcal{X}} \|\vect{f}(\vect{x}')\|_2 \Bigr] \\
        &= 2 \|\vect{\lambda} - \vect{\lambda}'\|_2 \; \mathbb{E}_n\left[ \max_{\vect{x}' \in \mathcal{X}} \|\vect{f}(\vect{x}')\|_2 \right].
    \end{align*}
    Since this holds with \(\vect{\lambda}\) and \(\vect{\lambda}'\) interchanged, we have established the inequality for \(\alpha_{\text{C-MOKG}}^n\).
    
    The proof for the residual uncertainty is similar. Again, let \(\vect{\lambda}, \vect{\lambda}' \in \Lambda\). Then for all \(n \in \mathbb{N}_0\),
    \begin{align*}
        H_n(\vect{\lambda}) -& H_n(\vect{\lambda}') \\
        &= \mathbb{E}_n\left[ \max_{\vect{x}' \in \mathcal{X}} \vect{\lambda} \cdot \vect{f}(\vect{x}') \right] - \max_{\vect{x}' \in \mathcal{X}} \mathbb{E}_n[\vect{\lambda} \cdot \vect{f}(\vect{x}')] \\
        &\quad - \mathbb{E}_n\left[ \max_{\vect{x}' \in \mathcal{X}} \vect{\lambda}' \cdot \vect{f}(\vect{x}') \right] + \max_{\vect{x}' \in \mathcal{X}} \mathbb{E}_n[\vect{\lambda}' \cdot \vect{f}(\vect{x}')] \\
        &\leq \mathbb{E}_n \left[ \max_{\vect{x}' \in \mathcal{X}} (\vect{\lambda} - \vect{\lambda}') \cdot \vect{f}(\vect{x}') \right]
        + \max_{\vect{x}' \in \mathcal{X}} \mathbb{E}_n \left[ (\vect{\lambda}' - \vect{\lambda}) \cdot \vect{f}(\vect{x}') \right] \\
        &\leq \mathbb{E}_n\left[ \max_{\vect{x}' \in \mathcal{X}} \|\vect{\lambda} - \vect{\lambda}'\|_2 \, \|\vect{f}(\vect{x}')\|_2 \right]
        + \max_{\vect{x}' \in \mathcal{X}} \mathbb{E}_n \bigl[ \|\vect{\lambda} - \vect{\lambda}'\|_2 \, \|\vect{f}(\vect{x}')\|_2 \bigr] \\
        &\leq 2 \|\vect{\lambda} - \vect{\lambda}'\|_2 \mathbb{E}_n\left[ \max_{\vect{x}' \in \mathcal{X}} \|\vect{f}(\vect{x}')\|_2 \right].
    \end{align*}
    Again, this holds with \(\vect{\lambda}\) and \(\vect{\lambda}'\) interchanged, which establishes the inequality for \(H_n\).
\end{proof}

We are now in a position to state and prove the first two main results, which show convergence of \(\alpha_{\text{C-MOKG}}^n(\vect{x}, m; \vect{\lambda}, \vect{c})\) to zero in a certain sense, both when using random scalarizations and expectation over scalarizations.

\begin{theorem}\label{thm:app-cmokg-r-unif-zero-liminf}
    Suppose we select samples using C-MOKG with random scalarization weights \(\vect{\lambda}_1, \vect{\lambda}_2, \dots\) chosen independently according to distribution \(p(\vect{\lambda})\).
    Then,
    \[
        \forall \vect{\lambda} \in \Lambda\quad
        \mathbb{P}\left( \liminf_{n \to \infty} \sup_{\vect{x} \in \mathcal{X},\, m \in [M]} \alpha_{\text{C-MOKG}}^n(\vect{x},m;\, \vect{\lambda}, \vect{c}) = 0 \right) = 1.
    \]
\end{theorem}
\begin{remark}
    In fact, it is true that for all \(\vect{\lambda} \in \Lambda\),
    \begin{equation*}
        \sup_{\vect{x} \in \mathcal{X},\, m \in [M]} \alpha_{\text{C-MOKG}}^n(\vect{x}, m;\, \vect{\lambda}, \vect{c}) \to 0
    \end{equation*}
    as \(n \to \infty\), almost surely. However, the proof of this will have to wait until \cref{thm:app-cmokg-r-unif-zero}.
\end{remark}

\begin{proof}
    Let \(\vect{\lambda} \in \Lambda\).
    The proof in the single objective case found in \cite{bect2009sur} rests on the fact that we use the same acquisition function at each step of the optimization. However, for the multi-objective case with random scalarizations \((\vect{\lambda}_n)_{n=1}^\infty\), we are effectively changing the acquisition function at each step. The key observation which lets us proceed is that, while we will not use the exact weights \(\vect{\lambda}\) infinitely often, we will use weights which are arbitrarily close, infinitely often, and these will be similar enough.
    Formally, we observe that there exists a (random) subsequence \((\vect{\lambda}_{1+n_j})_{j=0}^\infty\) with \(\vect{\lambda}_{1+n_j} \to \vect{\lambda}\) as \(j \to \infty\) almost surely. Without loss of generality and for the purpose of slightly easing notation later, we will assume \(n_0 = 0\).
    We will first show that \(\sup_{\vect{x} \in \mathcal{X},\, m \in [M]} \alpha_{\text{C-MOKG}}^{n_j}(\vect{x}, m;\, \vect{\lambda}_{1+n_j}, \vect{c}) \to 0\) almost surely, then use the convergence of the subsequence to assert the same when we replace \(\vect{\lambda}_{1+n_j}\) with \(\vect{\lambda}\).

    Inspired by the notation from \cite{bect2009sur}, for each \(j \in \mathbb{N}_0\) let
    \begin{alignat*}{2}
        \Delta^{(1)}_{j+1} &= H_{n_j}(\vect{\lambda}_{1+n_j}) - H_{n_{(j+1)}}(\vect{\lambda}_{1+n_{(j+1)}}),
        \qquad&
        \overline{\Delta}^{(1)}_{j+1} &= \mathbb{E}_{n_j}[\Delta^{(1)}_{j+1}], \\
        \Delta^{(2)}_{j+1} &= H_{n_{(j+1)}}(\vect{\lambda}_{1+n_{(j+1)}}) - H_{n_{(j+1)}}(\vect{\lambda}_{1+n_j}),
        \qquad&
        \overline{\Delta}^{(2)}_{j+1} &= \mathbb{E}_{n_j}[\Delta^{(2)}_{j+1}].
    \end{alignat*}
    The first pair here give the difference between successive terms in the sequence \((H_{n_j}(\vect{\lambda}_{1+n_j}))_{j=0}^\infty\) while the second is a correction to `undo' changing \(\vect{\lambda}_{1+n_j}\).
    Then, since \(H_n(\vect{\lambda})\) is a supermartingale (\cref{thm:app-Hn-supmart}), and applying Equations~\eqref{eq:app-anxn-1} and \eqref{eq:app-quasi-sur-1}, we have
    \begin{align*}
        \overline{\Delta}^{(1)}_{j+1} + \overline{\Delta}^{(2)}_{j+1}
        &= H_{n_j}(\vect{\lambda}_{1+n_j}) - \mathbb{E}_{n_j}[H_{n_{(j+1)}}(\vect{\lambda}_{1+n_j})] \\
        &\geq H_{n_j}(\vect{\lambda}_{1+n_j}) - \mathbb{E}_{n_j}[H_{1+n_j}(\vect{\lambda}_{1+n_j})] \\
        &= c_{m_{1+n_j}} \alpha_{\text{C-MOKG}}^{n_j}(\vect{x}_{1+n_j}, m_{1+n_j};\, \vect{\lambda}_{1+n_j}, \vect{c}) \\
        &> c_{m_{1+n_j}} \left( \sup_{\vect{x} \in \mathcal{X},\, m \in [M]} \alpha_{\text{C-MOKG}}^{n_j}(\vect{x}, m;\, \vect{\lambda}_{1+n_j}, \vect{c}) - \eta_{n_j} \right).
    \end{align*}
    Therefore, if we can show that almost surely \(\overline{\Delta}^{(1)}_j \to 0\) and \(\overline{\Delta}^{(2)}_j \to 0\) as \(j \to \infty\), then we will have almost surely \(\sup_{\vect{x} \in \mathcal{X},\, m \in [M]} \alpha_{\text{C-MOKG}}^{n_j}(\vect{x}, m;\, \vect{\lambda}_{1+n_j}, \vect{c}) \to 0\).

    For any \(j \in \mathbb{N}\) we have \(\sum_{i=1}^j \Delta^{(1)}_i = H_0(\vect{\lambda}_1) - H_{n_j}(\vect{\lambda}_{1+n_j})\).
    Therefore, since each \(H_{n_j}(\vect{\lambda}_{1+n_j})\) is almost surely non-negative (\cref{thm:app-Hn-supmart}),
    \begin{multline*}
        \forall j \in \mathbb{N}\quad
        \mathbb{E}\left[ \sum_{i=1}^j \overline{\Delta}^{(1)}_i \right]
        = \mathbb{E}\left[ \sum_{i=1}^j \Delta^{(1)}_i \right] \\
        = \mathbb{E}[H_0(\vect{\lambda}_1) - H_{n_j}(\vect{\lambda}_{1+n_j})]
        \leq \mathbb{E}[H_0(\vect{\lambda}_1)]
        < \infty.
    \end{multline*}
    Thus \(\overline{\Delta}^{(1)}_j \to 0\) as \(j \to \infty\) almost surely.

    To show that \(\Delta^{(2)}_j \to 0\) almost surely, we will use \cref{thm:app-cmokg-and-h-lipschitz}.
    Indeed, this gives that, for all \(j\),
    \begin{align*}
        \bigl|\Delta^{(2)}_{j+1}\bigr| &= \bigl| H_{n_{(j+1)}}(\vect{\lambda}_{1+n_{(j+1)}}) - H_{n_{(j+1)}}(\vect{\lambda}_{1+n_j}) \bigr| \\
        &\leq 2 \, \|\vect{\lambda}_{1+n_{(j+1)}} - \vect{\lambda}_{1+n_j}\|_2 \; \mathbb{E}_{n_{(j+1)}}\!\left[ \max_{\vect{x}' \in \mathcal{X}} \|\vect{f}(\vect{x}')\|_2 \right].
    \end{align*}
    But \(\mathbb{E}_{n_{(j+1)}}[\max_{\vect{x}' \in \mathcal{X}} \|\vect{f}(\vect{x}')\|_2] \to \mathbb{E}_\infty[\max_{\vect{x}' \in \mathcal{X}} \|\vect{f}(\vect{x}')\|_2]\) as \(j \to \infty\) almost surely, by L\'evy's zero-one law. Also, by construction, \(\|\vect{\lambda}_{1 + n_{(j+1)}} - \vect{\lambda}_{1 + n_j}\|_2 \to 0\) as \(j \to \infty\) almost surely. Therefore, \(\Delta^{(2)}_j \to 0\) almost surely.

    Whence,
    \begin{equation*}
        0 \leq \sup_{\vect{x} \in \mathcal{X},\, m \in [M]} \alpha_{\text{C-MOKG}}^{n_j}(\vect{x}, m;\, \vect{\lambda}_{1+n_j}, \vect{c})
        < \frac{1}{c_{m_{1+n_j}}} (\Delta^{(1)}_{j+1} + \Delta^{(2)}_{j+1}) + \eta_{n_j} \to 0
    \end{equation*}
    as \(j \to \infty\) almost surely.

    To show that we can replace the \(\vect{\lambda}_{1 + n_j}\) with their limit \(\vect{\lambda}\), observe that for all \(j \geq 0\),
    \begin{align*}
        \sup_{\vect{x} \in \mathcal{X},\, m \in [M]} &\alpha_{\text{C-MOKG}}^{n_j}(\vect{x}, m;\, \vect{\lambda}, \vect{c}) \\
        &\leq \sup_{\vect{x} \in \mathcal{X},\, m \in [M]} \alpha_{\text{C-MOKG}}^{n_j}(\vect{x}, m;\, \vect{\lambda}_{1+n_j}, \vect{c}) \\
        &\quad + \sup_{\vect{x} \in \mathcal{X},\, m \in [M]} \Bigl( \alpha_{\text{C-MOKG}}^{n_j}(\vect{x}, m;\, \vect{\lambda}, \vect{c}) - \alpha_{\text{C-MOKG}}^{n_j}(\vect{x}, m;\, \vect{\lambda}_{1+n_j}, \vect{c}) \Bigr) \\
        &\leq \underbrace{\sup_{\vect{x} \in \mathcal{X},\, m \in [M]} \alpha_{\text{C-MOKG}}^{n_j}(\vect{x}, m;\, \vect{\lambda}_{1+n_j}, \vect{c})}_{\to 0} \\
        &\quad + 2\,\underbrace{\|\vect{\lambda} - \vect{\lambda}_{1+n_j}\|_2}_{\to 0} \, \underbrace{\mathbb{E}_{n_j}\!\left[\max_{\vect{x}' \in \mathcal{X}} \|\vect{f}(\vect{x}')\|_2 \right]}_{\to \mathbb{E}_\infty[\max_{\vect{x}' \in \mathcal{X}} \|\vect{f}(\vect{x}')\|_2] < \infty } \\
        &\to 0 \quad\text{as } j \to \infty.
    \end{align*}
    On the penultimate line, we have used that \(\max_{\vect{x}' \in \mathcal{X}} \|\vect{f}(\vect{x}')\|_2\) is \(L^1\)-integrable and thus the conditional expectations, \(\mathbb{E}_{n_j}[\max_{\vect{x}' \in \mathcal{X}} \|\vect{f}(\vect{x}')\|_2]\), converge almost surely to \(\mathbb{E}_\infty[\max_{\vect{x}' \in \mathcal{X}} \|\vect{f}(\vect{x}')\|_2] < \infty\) by L\'evy's zero-one law.
    
    Therefore, the whole upper bound converges to zero and we conclude that
    \[\sup_{\vect{x} \in \mathcal{X},\, m \in [M]} \alpha_{\text{C-MOKG}}^{n_j}(\vect{x}, m;\, \vect{\lambda}, \vect{c}) \to 0 \quad\text{as}\quad j \to \infty \quad\text{a.s.}.\]
    Hence, \(\liminf_{n \to \infty} \sup_{\vect{x} \in \mathcal{X},\, m \in [M]} \alpha_{\text{C-MOKG}}^n(\vect{x}, m;\, \vect{\lambda}, \vect{c}) = 0\) almost surely, completing the proof.
\end{proof}

We also make a similar statement for expectation over scalarization, except for this one we can already replace the limit inferior with a true limit.
\begin{theorem}\label{thm:app-cmokg-e-unif-zero}
    Suppose we select samples using C-MOKG with expectation over scalarizations.
    Then, for all preference vectors \(\vect{\lambda} \in \Lambda\), \(\alpha_{\text{C-MOKG}}^n(\vect{x}, m;\, \vect{\lambda}, \vect{c}) \to 0\) as \(n \to \infty\) uniformly in \(\vect{x} \in \mathcal{X}\) and \(m \in \{1, \dots, M\}\), almost surely. That is,
    \[\forall \vect{\lambda} \in \Lambda\quad \mathbb{P}\left(\sup_{\vect{x} \in \mathcal{X},\, m \in [M]} \alpha_{\text{C-MOKG}}^n(\vect{x}, m;\, \vect{\lambda}, \vect{c}) \to 0 \quad\text{as}\quad n \to \infty \right) = 1.\]
\end{theorem}
\begin{proof}
The proof in the case of expectation over scalarizations is simpler than that of \cref{thm:app-cmokg-r-unif-zero-liminf} since we have used the same acquisition function at each step of the optimization.

We will first prove that \(\overline{\alpha}_{\text{C-MOKG}}^n(\vect{x}, m; \vect{c}) \to 0\) uniformly in \(\vect{x} \in \mathcal{X}\) and \(m \in \{1, 2, \dots, M\}\), almost surely.
For each \(n \in \mathbb{N}_0\), let us redefine \(\Delta_{n+1} = \overline{H}_n - \overline{H}_{n+1}\) and \(\overline{\Delta}_{n+1} = \mathbb{E}_n[\Delta_{n+1}]\).
As remarked earlier in Equation~\eqref{eq:app-anxn-2}, \(\overline{\alpha}_{\text{C-MOKG}}^n(\vect{x}_{n+1}, m_{n+1};\, \vect{c}) = \frac{1}{c_{m_{n+1}}} (\overline{H}_n - \mathbb{E}_n[\overline{H}_{n+1}])\).
Therefore, using that \(\vect{x}_{n+1}\) and \(m_{n+1}\) were chosen to maximize \(\overline{\alpha}_{\text{C-MOKG}}^n\) according to Equation~\eqref{eq:app-quasi-sur-2},
\begin{multline} \label{eq:app-cmokg-e-unif-zero-delta-bound}
    \frac{1}{c_{m_{n+1}}}\overline{\Delta}_{n+1} = \frac{1}{c_{m_{n+1}}}\left(\overline{H}_n - \mathbb{E}_n\bigl[\,\overline{H}_{n+1}\bigr] \right) \\
    = \overline{\alpha}_{\text{C-MOKG}}^n(\vect{x}_{n+1}, m_{n+1};\, \vect{c})
    > \sup_{\vect{x} \in \mathcal{X},\, m \in [M]} \overline{\alpha}_{\text{C-MOKG}}^n(\vect{x}, m; \vect{c}) - \eta_n.
\end{multline}
Our problem is now reduced to showing that \(\overline{\Delta}_n \to 0\) as \(n \to \infty\).
For any \(n \in \mathbb{N}_0\) we have \(\sum_{j=1}^n \Delta_j = \overline{H}_0 - \overline{H}_n\).
Therefore,
\begin{equation*}
    \forall n \in \mathbb{N}_0\quad
    \mathbb{E}\left[ \sum_{j=1}^n \overline{\Delta}_j \right]
    = \mathbb{E}\left[ \sum_{j=1}^n \Delta_j \right]
    = \mathbb{E}[\overline{H}_0 - \overline{H}_n]
    \leq \mathbb{E}[\overline{H}_0]
    < \infty
\end{equation*}
where we used \cref{thm:app-Hn-supmart} for the penultimate inequality.
Also observe that combining Equation~\eqref{eq:app-cmokg-e-unif-zero-delta-bound} with \cref{thm:app-cmokg-nonneg} implies
\begin{equation*}
    \forall n \in \mathbb{N}_0\quad \overline{\Delta}_{n+1} = c_{m_{n+1}} \overline{\alpha}_{\text{C-MOKG}}^n(\vect{x}_{n+1}, m_{n+1}; \vect{c}) \geq 0
\end{equation*}
so \(\overline{\Delta}_{n+1} \to 0\) as \(n \to \infty\).
Hence,
\[0 \leq \sup_{\vect{x} \in \mathcal{X},\, m \in [M]} \overline{\alpha}_{\text{C-MOKG}}^n(\vect{x}, m; \vect{c}) < \frac{1}{c_{m_{n+1}}}\overline{\Delta}_{n+1} + \eta_n \to 0\]
as \(n \to \infty\) almost surely.

We will now use this to show that for all \(\vect{\lambda} \in \Lambda\), \(\alpha_{\text{C-MOKG}}^n(\vect{x}, m;\, \vect{\lambda}, \vect{c}) \to 0\) uniformly in \(\vect{x} \in \mathcal{X}\) and \(m \in \{1, \dots, M\}\), almost surely.
Suppose for contradiction that this does not hold. Then there exists \(\vect{\lambda} \in \Lambda\) such that with non-zero probability
\[\limsup_{n \to \infty} \sup_{\vect{x}\in\mathcal{X},\, m \in [M]} \alpha_{\text{C-MOKG}}^n(\vect{x}, m;\, \vect{\lambda}, \vect{c}) > 0.\]
By \cref{thm:app-cmokg-and-h-lipschitz}, for all \(\vect{\lambda}' \in \Lambda\), and all \(\vect{x} \in \mathcal{X}\), \(m \in \{1, \dots, M\}\) and \(n \in \mathbb{N}_0\),
\[\alpha_{\text{C-MOKG}}^n(\vect{x}, m;\, \vect{\lambda}', \vect{c}) \geq \alpha_{\text{C-MOKG}}^n(\vect{x}, m;\, \vect{\lambda}, \vect{c}) - \frac{2}{c_m} \|\vect{\lambda} - \vect{\lambda}'\|_2 \,\mathbb{E}_n\!\left[ \max_{\vect{x}' \in \mathcal{X}} \|\vect{f}(\vect{x}')\|_2 \right].\]

Now, \(\max_{\vect{x}' \in \mathcal{X}} \|\vect{f}(\vect{x}')\|_2\) is \(L^1\)-integrable and so \(\mathbb{E}_n\left[ \max_{\vect{x}' \in \mathcal{X}} \|\vect{f}(\vect{x}')\|_2 \right] \to \mathbb{E}_\infty\left[ \max_{\vect{x}' \in \mathcal{X}} \|\vect{f}(\vect{x}')\|_2 \right]\) as \(n \to \infty\) almost surely, by L\'evy's zero-one law.
Therefore, there exists a random \(N \in \mathbb{N}_0\) such that, almost surely, for all \(n \geq N\), \(\mathbb{E}_n\left[ \max_{\vect{x}' \in \mathcal{X}} \|\vect{f}(\vect{x}')\|_2 \right] \leq \frac{3}{2} \mathbb{E}_\infty\left[ \max_{\vect{x}' \in \mathcal{X}} \|\vect{f}(\vect{x}')\|_2 \right]\).
That is, we have an upper bound on \(\mathbb{E}_n[\max_{\vect{x}' \in \mathcal{X}} \|\vect{f}(\vect{x}')\|_2]\) which is independent of \(n\).
Thus there exists a small \(L^2\)-ball \(B_{\vect{\lambda}} \subset \Lambda\) centered on \(\vect{\lambda}\) with a strictly positive, random radius such that for all \(\vect{\lambda}' \in B_{\vect{\lambda}}\) and \(n \geq N\) we have
\[\frac{2}{c_m} \|\vect{\lambda} - \vect{\lambda}'\|_2 \,\mathbb{E}_n\!\left[ \max_{\vect{x}' \in \mathcal{X}} \|\vect{f}(\vect{x}')\|_2 \right] < \min(A_{\vect{\lambda}}, 1)\]
where \(A_{\vect{\lambda}} = \frac{1}{2} \limsup_{n \to \infty} \sup_{\vect{x} \in \mathcal{X},\, m \in [M]} \alpha_{\text{C-MOKG}}^n(\vect{x}, m;\, \vect{\lambda}, \vect{c}) > 0\). We have taken the minimum with 1 here to cover the case where \(A_{\vect{\lambda}} = +\infty\).
Thus, for all \(n \geq N\) and \(\vect{\lambda}' \in B_{\vect{\lambda}}\), and all \(\vect{x} \in \mathcal{X}\) and \(m \in \{1, \dots, M\}\),
\[\alpha_{\text{C-MOKG}}^n(\vect{x}, m;\, \vect{\lambda}', \vect{c}) \geq \alpha_{\text{C-MOKG}}^n(\vect{x}, m;\, \vect{\lambda}, \vect{c}) - \min(A_{\vect{\lambda}}, 1).\]
Since C-MOKG is non-negative (\cref{thm:app-cmokg-nonneg}), restricting to \(B_{\vect{\lambda}}\) and integrating gives a lower bound on the expectation for all \(\vect{x} \in \mathcal{X}\), \(m \in \{1, \dots, M\}\) and all \(n \geq N\),
\[\overline{\alpha}_{\text{C-MOKG}}^n(\vect{x}, m;\, \vect{c}) \geq \mathbb{P}(B_{\vect{\lambda}})(\alpha_{\text{C-MOKG}}^n(\vect{x}, m;\, \vect{\lambda}, \vect{c}) - \min(A_{\vect{\lambda}}, 1)).\]
Then, taking the supremum and limit gives
\begin{align*}
    \limsup_{n \to \infty} &\sup_{\vect{x} \in \mathcal{X},\, m \in [M]} \overline{\alpha}_{\text{C-MOKG}}^n(\vect{x}, m;\, \vect{c}) \\
    &\geq \mathbb{P}(B_{\vect{\lambda}}) \left( \limsup_{n \to \infty} \sup_{\vect{x} \in \mathcal{X},\, m \in [M]} \alpha_{\text{C-MOKG}}^n(\vect{x}, m;\, \vect{\lambda}, \vect{c}) - \min(A_{\vect{\lambda}}, 1) \right) \\
    &\geq \frac{1}{2} \mathbb{P}(B_{\vect{\lambda}}) \; \limsup_{n \to \infty} \sup_{\vect{x} \in \mathcal{X},\, m \in [M]} \alpha_{\text{C-MOKG}}^n(\vect{x}, m;\, \vect{\lambda}, \vect{c}) \\
    &> 0.
\end{align*}
By choice of \(\vect{\lambda}\), and because the distribution over \(\Lambda\) is strictly positive, the final strict inequality holds with non-zero probability. Hence we have a contradiction and so we conclude that for all \(\vect{\lambda} \in \Lambda\), \(\sup_{\vect{x} \in \mathcal{X}, m \in [M]} \alpha_{\text{C-MOKG}}^n(\vect{x}, m;\, \vect{\lambda}, \vect{c}) \to 0\) as \(n \to \infty\) almost surely.
\end{proof}

\subsection{Convergence to Optimal Recommendations}
We have shown that, for each of the two acquisition functions, the limit inferior of \(\alpha_{\text{C-MOKG}}^n(\,\cdot\,, \cdot\,;\, \vect{\lambda}, \vect{c})\), is zero almost surely for each \(\vect{\lambda} \in \Lambda\). In fact, for \(\overline{\alpha}_{\text{C-MOKG}}^n\) we have shown almost sure, uniform convergence to 0. However, our aim is a much stronger result.

Recall Equation~\eqref{eq:kg-recommendations} from the main text, which we restate here as Equation~\eqref{eq:app-cmokg-recommendations}. For each \(n \in \mathbb{N}_0\) and each preference vector \(\vect{\lambda} \in \Lambda\), let
\begin{equation} \label{eq:app-cmokg-recommendations}
    \vect{x}_{n, \vect{\lambda}}^* \in \argmax_{\vect{x} \in \mathcal{X}} \mathbb{E}_n[\vect{\lambda} \cdot \vect{f}(\vect{x})]
\end{equation}
be a random variable which maximizes the posterior mean of the scalarized objective at stage \(n\). Thus, assuming no model mismatch, \(\vect{\lambda} \cdot \vect{f}(\vect{x}_{1, \vect{\lambda}}^*),\, \vect{\lambda} \cdot \vect{f}(\vect{x}_{2, \vect{\lambda}}^*), \dots\) is the sequence of (noiseless) scalarized objective values we would obtain if we were to use the recommended point at each stage of the optimization. Presented for the single-objective case as Proposition 4.9 in \cite{bect2009sur}, the next theorem tells us that, this sequence converges to the true maximum of the scalarized, hidden objective function, \(\vect{\lambda} \cdot \vect{f}\).
\setcounter{tmpcounter}{\value{theorem}}
\setcounter{theorem}{1}
\begin{theorem}[Restated from main text]\label{thm:app-consistency}
    Define the \(\vect{x}_{n, \vect{\lambda}}^*\) as in Equation~\eqref{eq:app-cmokg-recommendations}.
    When using C-MOKG with either random scalarizations, or expectation over scalarizations, we have
    \[\forall \vect{\lambda} \in \Lambda \quad \vect{\lambda} \cdot \vect{f}(\vect{x}_{n, \vect{\lambda}}^*) \to \max_{\vect{x} \in \mathcal{X}} \vect{\lambda} \cdot \vect{f}(\vect{x}) \quad\text{as}\quad n \to \infty\]
    almost surely and in mean.
\end{theorem}
\setcounter{theorem}{\value{tmpcounter}}

Before showing this, we will first prove an important result establishing almost sure, uniform convergence of the posterior mean and covariance functions of the GP surrogate model, regardless of the acquisition function used. This will be invaluable in proving \cref{thm:app-consistency}.

Let \(\mathcal{F}_\infty = \sigma(\cup_{n=0}^\infty \mathcal{F}_n)\) be the smallest \(\sigma\)-algebra containing all the \(\mathcal{F}_n\) and let \(\mathbb{E}_\infty[\,\cdot\,] = \mathbb{E}[\,\cdot\, |\, \mathcal{F}_\infty]\) denote expectation conditional on all observations and choices.
Also, for each \(n \in \mathbb{N}_0 \cup \{\infty\}\), define the conditional mean and covariance functions of \(\vect{f}\),
\begin{align*}
    \vect{\mu}_n&: \mathcal{X} \to \mathbb{R}^M, \quad \vect{x} \mapsto \mathbb{E}_n[\vect{f}(\vect{x})], \\
    \matr{K}_n&: \mathcal{X} \times \mathcal{X} \to \mathbb{R}^{M \times M}, \quad (\vect{x}, \vect{x}') \mapsto \mathbb{E}_n[\vect{f}(\vect{x}) \vect{f}(\vect{x}')^T] - \mathbb{E}_n[\vect{f}(\vect{x})] \mathbb{E}_n[\vect{f}(\vect{x}')]^T. 
\end{align*}
Note that these are vector- and matrix-valued stochastic processes.

\begin{proposition} \label{thm:app-unif-convergence-cond-gps}
    For any choice of query locations \((\vect{x}_n, m_n)_{n=1}^\infty\), the sequences of stochastic processes \(\vect{\mu}_n \to \vect{\mu}_\infty\) and \(\matr{K}_n \to \matr{K}_\infty\) converge uniformly as \(n \to \infty\), both almost surely and in \(L^p\) for all \(1 \leq p < \infty\). Furthermore, the limits \(\vect{\mu}_\infty\) and \(\matr{K}_\infty\) are continuous.
\end{proposition}
\begin{proof}
Since \(\mathcal{X}\) is compact, the space \(\mathcal{C}(\mathcal{X}, \mathbb{R}^M)\) of continuous functions \(\mathcal{X} \to \mathbb{R}^M\) forms a Banach space when equipped with the supremum norm \(\| \cdot \|_\infty\). The multi-output Gaussian process \(\vect{f}\) has continuous sample paths and so can be viewed as a random element of this space. It is also \(L^p\)-integrable for any \(1 \leq p < \infty\). Indeed, all moments of \(\sup_{\vect{x} \in \mathcal{X},\, m \in [M]} |f_m(\vect{x})|\) are finite%
\footnote{see for example Equation~2.34 in \cite{azais2009levelsets} and view \(\vect{f}\) as an \(\mathbb{R}\)-valued GP on input space \(\mathcal{X} \times [M]\)} and so
\[\|\vect{f}\|_{L^p(\mathcal{C}(\mathcal{X}, \mathbb{R}^M))} = \mathbb{E}\left[\|\vect{f}\|_\infty^p\right]^{1/p} = \mathbb{E}\left[ \left( \sup_{\vect{x}\in\mathcal{X},\, m \in [M]} |f_m(\vect{x})| \right)^p\, \right]^{1/p} < \infty.\]
Observe that the conditional means \(\vect{\mu}_n = \mathbb{E}_n[\vect{f}]\) and \(\vect{\mu}_\infty = \mathbb{E}_\infty[\vect{f}]\) are continuous since Banach spaces are closed under taking conditional expectation (see e.g. Proposition 1.10 from \cite{pisier2016martingalesbanach}). In fact, we have shown that they are a martingale of the form in Theorems 1.14 and 1.30 in \cite{pisier2016martingalesbanach}, and so \(\vect{\mu}_n \to \vect{\mu}_\infty\) almost surely and in \(L^p\). This is a convergence in the function space \(\mathcal{C}(\mathcal{X}, \mathbb{R}^M)\) using the supremum norm, which is equivalent to saying that the convergence of the processes is uniform. That is, \(\vect{\mu}_n \to \vect{\mu}_\infty\) uniformly, both almost surely and in \(L^p\) for any \(1 \leq p < \infty\).

We may use the same argument on the sequence of second moments of \(\vect{f}\) in the Banach space \(\mathcal{C}(\mathcal{X} \times \mathcal{X},\, \mathbb{R}^{M \times M})\). Let \(\matr{M}_n^{(2)}(\vect{x}, \vect{x}') = \mathbb{E}_n[\vect{f}(\vect{x}) \vect{f}(\vect{x}')^T]\) denote the second moments of \(\vect{f}\). The process \(\matr{F}^{(2)} \in \mathcal{C}(\mathcal{X} \times \mathcal{X},\, \mathbb{R}^{M \times M})\) defined by \(\matr{F}^{(2)}(\vect{x}, \vect{x'}) = \vect{f}(\vect{x})\vect{f}(\vect{x}')^T\) is \(L^p\)-integrable for all \(1 \leq p < \infty\) since
\begin{multline*}
    \|\matr{F}^{(2)}\|_{L^p(\mathcal{C}(\mathcal{X} \times \mathcal{X},\, \mathbb{R}^{M \times M}))}
    = \mathbb{E}\biggl[ \Bigl( \sup_{\substack{\vect{x}, \vect{x}' \in \mathcal{X} \\ m, m' \in [M]}} |f_m(\vect{x})f_{m'}(\vect{x}')| \Bigr)^p\, \biggr]^{1/p} \\
    = \mathbb{E}\biggl[ \Bigl( \sup_{\vect{x} \in \mathcal{X},\, m \in [M]} |f_m(\vect{x})| \Bigr)^{2p} \biggr]^{1/p}
    = \|\vect{f}\|_{L^{2p}(\mathcal{C}(\mathcal{X},\, \mathbb{R}^M))}^2
    < \infty.
\end{multline*}
Therefore, we may again apply martingale convergence Theorems 1.14 and 1.30 from \cite{pisier2016martingalesbanach} to conclude that \(\matr{M}_n^{(2)} \to \matr{M}_\infty^{(2)}\) almost surely and in \(L^p\), where we have written \(\matr{M}_\infty^{(2)}(\vect{x}, \vect{x}') = \mathbb{E}_\infty[\vect{f}(\vect{x}) \vect{f}(\vect{x}')^T]\).
Since
\[\forall n \in \mathbb{N}\cup\{\infty\}\; \forall \vect{x}, \vect{x}' \in \mathcal{X} \quad \matr{K}_n(\vect{x}, \vect{x}') = \matr{M}_n^{(2)}(\vect{x}, \vect{x}') - \vect{\mu}_n(\vect{x})\vect{\mu}_n(\vect{x}')^T,\]
we conclude that \(\matr{K}_n \to \matr{K}_\infty\) uniformly, both almost surely and in \(L^p\) for all \(1 \leq p < \infty\).
By the same argument as before, the \(\matr{M}_n^{(2)}\) and \(\matr{M}_\infty^{(2)}\) are continuous as conditional expectations of random elements in the Banach space of continuous functions. So the \(\matr{K}_n\) and \(\matr{K}_\infty\) are also continuous.
\end{proof}

Our strategy to prove \cref{thm:app-consistency} will be to first show that the limiting estimate, \(\vect{\mu}_\infty\), approximates the objective \(\vect{f}\) up to a constant almost surely. That is, we will show that \(\vect{f} - \vect{\mu}_\infty\) has almost surely constant sample paths.
The intuition then is that knowledge of \(\vect{\mu}_\infty\) is sufficient to determine the location of the maximum of \(\vect{f}\). To show that \(\vect{f} - \vect{\mu}_\infty\) has almost surely constant sample paths, we will first show that the limiting covariance function, \(\matr{K}_\infty\), has almost surely constant sample paths.

\begin{lemma} \label{thm:app-const-sample-paths-k}
The limiting covariance function \(\matr{K}_\infty: \mathcal{X} \times \mathcal{X} \to \mathbb{R}^{M \times M}\) has almost surely constant sample paths. That is,
\[\mathbb{P}\bigl(\forall \vect{x}, \vect{u}, \vect{x}', \vect{u}' \in \mathcal{X}\quad \matr{K}_\infty(\vect{u}, \vect{x}) = \matr{K}_\infty(\vect{u}', \vect{x}')\bigr) = 1.\]
\end{lemma}
\begin{proof}
    This proof makes use of the fact that the scalarizations are linear and thus commute with the expectation operator.

    Let \(\vect{x} \in \mathcal{X}\), \(m \in \{1, \dots, M\}\) and \(\vect{\lambda} \in \Lambda\).
    For each \(n \in \mathbb{N}_0\) and \(\vect{u} \in \mathcal{X}\) let
    \begin{align*}
        W_{n, \vect{u}} &= \mathbb{E}_n[\vect{\lambda} \cdot \vect{f}(\vect{u}) \,|\, y_m^{\vect{x}}] - \mathbb{E}_n[\vect{\lambda} \cdot \vect{f}(\vect{x}_{n, \vect{\lambda}}^*) \,|\, y_m^{\vect{x}}] \\
        &= \vect{\lambda} \cdot  \mathbb{E}_n[\vect{f}(\vect{u}) \,|\, y_m^{\vect{x}}] - \vect{\lambda} \cdot \mathbb{E}_n[\vect{f}(\vect{x}_{n, \vect{\lambda}}^*) \,|\, y_m^{\vect{x}}].
    \end{align*}
    Observe that \(\alpha_{\text{C-MOKG}}^n(\vect{x}, m;\, \vect{\lambda}, \vect{c}) = \frac{1}{c_m} \mathbb{E}_n[\sup_{\vect{x}' \in \mathcal{X}} W_{n, \vect{x}'}]\).
    
    Let \(n \in \mathbb{N}_0\) and \(\vect{u} \in \mathcal{X}\).
    Since \(\sup_{\vect{x}' \in \mathcal{X}} W_{n, \vect{x}'} \geq 0\), we have
    \begin{gather*}
        \sup_{\vect{x}' \in \mathcal{X}} W_{n, \vect{x}'} = \sup_{\vect{x}' \in \mathcal{X}} \max(W_{n, \vect{x}'}, 0) \\
        \begin{split}
            \Rightarrow\qquad
            0 \leq \mathbb{E}_n[\max(W_{n, \vect{u}}, 0)]
            \leq \mathbb{E}_n\!\left[ \sup_{\vect{x}' \in \mathcal{X}} \max(W_{n, \vect{x}'}, 0) \right] \\
            = c_m \alpha_{\text{C-MOKG}}^n(\vect{x}, m;\, \vect{\lambda}, \vect{c}).
        \end{split}
    \end{gather*}
    Applying \cref{thm:app-cmokg-r-unif-zero-liminf} or \cref{thm:app-cmokg-e-unif-zero} depending on the acquisition strategy, we have almost surely \(\liminf_{n \to \infty} \sup_{\vect{x} \in \mathcal{X},\, m \in [M]} \alpha_{\text{C-MOKG}}^n(\vect{x}, m;\, \vect{\lambda}, \vect{c}) = 0\). Therefore, \(\liminf_{n \to \infty} \mathbb{E}_n[\max(W_{n, \vect{u}}, 0)] = 0\) almost surely.
    This implies that, almost surely,
    \begin{equation}\label{eq:app-w-converges-in prob}
        \forall \delta > 0 \quad \liminf_{n \to \infty} \mathbb{P}_n(W_{n, \vect{u}} > \delta) = 0,
    \end{equation}
    a statement closely resembling the definition of convergence in probability.
    Here, the (random) distribution \(\mathbb{P}_n\) is defined as the conditional expectation of the indicator variables. That is, for every event \(A \in \mathcal{F}\), \(\mathbb{P}_n(A) = \mathbb{E}_n[\mathbb{I}_A]\).
    Writing \(\vect{k}_{\infty, :, m}\) for the \(m\)\textsuperscript{th} column of \(\matr{K}_\infty\), we will use Equation~\eqref{eq:app-w-converges-in prob} to show that
    \begin{equation} \label{eq:app-kinfty-const-interim}
        \vect{\lambda} \cdot \vect{k}_{\infty,:,m}(\vect{u}, \vect{x}) = \vect{\lambda} \cdot \vect{k}_{\infty,:,m}(\vect{x}, \vect{x})
    \end{equation}
    almost surely -- something made significantly easier by noting that \(W_{n, \vect{u}}\) is a Gaussian variable.

    Write \(k_{n,m,m}\) for the \((m, m)\)\textsuperscript{th} element of \(\matr{K}_n\).
    If \(k_{n,m,m}(\vect{x}, \vect{x}) = 0\) then \(k_{n',m,m}(\vect{x}, \vect{x}) = 0\) for all \(n' \geq n\) and, by Cauchy-Schwarz, \(k_{n',m',m}(\vect{u}, \vect{x}) = 0\) for all \(m' \in \{1, \dots, M\}\) and all \(n' \geq n\) as well. Hence, \(\vect{k}_{n,:,m}(\vect{u}, \vect{x}) - \vect{k}_{n,:,m}(\vect{x}, \vect{x}) = 0\) for all sufficiently large \(n\) and so applying \cref{thm:app-unif-convergence-cond-gps} we have established Equation~\eqref{eq:app-kinfty-const-interim}.
    
    So we focus on the event where \(k_{n,m,m}(\vect{x}, \vect{x}) > 0\) for all \(n\).
    Using the formula for the conditional mean of a Gaussian process in Equation~\eqref{eq:gp-posterior-mean}, we can express \(W_{n, \vect{u}}\) as
    \begin{multline*}
        W_{n, \vect{u}} = \vect{\lambda} \cdot \Biggl( \vect{\mu}_n(\vect{u}) - \vect{\mu}_n(\vect{x}_{n, \vect{\lambda}}^*) \\
        + \frac{\vect{k}_{n, :, m}(\vect{u}, \vect{x}) - \vect{k}_{n, :, m}(\vect{x}_{n, \vect{\lambda}}^*,\, \vect{x})}{k_{n, m, m}(\vect{x}, \vect{x}) + \sigma_m^2} (y_m^{\vect{x}} - \mu_{n,m}(\vect{x})) \Biggr)
    \end{multline*}
    where \(y_m^{\vect{x}} \sim \mathcal{N}(\mu_{n, m}(\vect{x}),\, k_{n,m,m}(\vect{x}, \vect{x}) + \sigma_m^2)\) is the (random) value observed when we sample the \(m\)\textsuperscript{th} objective at \(\vect{x}\).
    Hence, conditional on \(\mathcal{F}_n\), the variable \(W_{n, \vect{u}}\) itself follows a normal distribution with mean and covariance given by
    \begin{align*}
        \mathbb{E}_n[W_{n, \vect{u}}] &= \vect{\lambda} \cdot (\vect{\mu}_n(\vect{u}) - \vect{\mu}_n(\vect{x}_{n, \vect{\lambda}}^*)), \\
        \mathrm{Var}_n[W_{n, \vect{u}}] &= \frac{( \vect{\lambda} \cdot [\vect{k}_{n, :, m}(\vect{u}, \vect{x}) - \vect{k}_{n, :, m}(\vect{x}_{n, \vect{\lambda}}^*, \vect{x})])^2}{k_{n, m, m}(\vect{x}, \vect{x}) + \sigma_m^2}.
    \end{align*}
    Whence, for any \(\delta > 0\),
    \begin{multline*}
        \mathbb{P}_n(W_{n, \vect{u}} > \delta) \\
        = 1 - \Phi\left( \frac{\sqrt{k_{n, m, m}(\vect{x}, \vect{x}) + \sigma_m^2}}{|\vect{\lambda} \cdot (\vect{k}_{n, :, m}(\vect{u}, \vect{x}) - \vect{k}_{n, :, m}(\vect{x}_{n, \vect{\lambda}}^*, \vect{x}))|} [\delta - \vect{\lambda} \cdot (\vect{\mu}_{n}(\vect{u}) - \vect{\mu}_n(\vect{x}_{n, \vect{\lambda}}^*))]\right)
    \end{multline*}
    where \(\Phi\) is the cumulative density function of a standard normal variable.
    
    We have from \cref{thm:app-unif-convergence-cond-gps} that \(\vect{\mu}_n \to \vect{\mu}_\infty\) uniformly as \(n \to \infty\), almost surely.
    Further, \(\vect{\mu}_\infty\) is continuous and \(\mathcal{X}\) is compact. Therefore, there exists a negative, random variable \(a\) (not depending on \(n\)) such that, for sufficiently large \(n\), \(\mathbb{E}_n[W_{n, \vect{u}}] = \vect{\lambda} \cdot (\vect{\mu}_n(\vect{u}) - \vect{\mu}_n(\vect{x}_{n, \vect{\lambda}}^*)) > a\) almost surely. For example. take any \(a < \min \vect{\lambda} \cdot \vect{\mu}_\infty - \max \vect{\lambda} \cdot \vect{\mu}_\infty\).
    Thus,
    \[\mathbb{P}_n(W_{n, \vect{u}} > \delta) \geq 1 - \Phi\left( \frac{\sqrt{k_{n, m, m}(\vect{x}, \vect{x}) + \sigma_m^2}}{|\vect{\lambda} \cdot (\vect{k}_{n, :, m}(\vect{u}, \vect{x}) - \vect{k}_{n, :, m}(\vect{x}_{n, \vect{\lambda}}^*, \vect{x}))|} (\delta - a) \right).\]
    By Equation~\eqref{eq:app-w-converges-in prob}, there exists a (random) subsequence \((n_j)_{j=0}^\infty\) such that the probabilities \(\mathbb{P}_{n_j}(W_{n_j, \vect{u}} > \delta) \to 0\) as \(j \to \infty\).
    Therefore, on this subsequence, \(\frac{\sqrt{k_{n_j, m, m}(\vect{x}, \vect{x}) + \sigma_m^2}}{|\vect{\lambda} \cdot (\vect{k}_{n_j, :, m}(\vect{u}, \vect{x}) - \vect{k}_{n_j, :, m}(\vect{x}_{n_j, \vect{\lambda}}^*, \vect{x}))|} \to \infty\) as \(j \to \infty\).
    Since \((k_{n,m,m}(\vect{x}, \vect{x}))_{n=1}^\infty\) is a decreasing sequence, we must then have \(\vect{\lambda} \cdot (\vect{k}_{n_j, :, m}(\vect{u}, \vect{x}) - \vect{k}_{n_j, :, m}(\vect{x}_{n_j, \vect{\lambda}}^*, \vect{x})) \to 0\) almost surely.
    This holds for all \(\vect{u}\), including the case \(\vect{u} = \vect{x}\), and so for all \(\vect{u} \in \mathcal{X}\),
    \begin{multline*}
        \vect{\lambda} \cdot (\vect{k}_{n_j, :, m}(\vect{u}, \vect{x}) - \vect{k}_{n_j, :, m}(\vect{x}, \vect{x})) = \\
        \biggl(\vect{\lambda} \cdot (\vect{k}_{n_j, :, m}(\vect{u}, \vect{x}) - \vect{k}_{n_j, :, m}(\vect{x}_{n_j, \vect{\lambda}}^*, \vect{x})) \\
        - \vect{\lambda} \cdot (\vect{k}_{n_j, :, m}(\vect{x}, \vect{x}) - \vect{k}_{n_j, :, m}(\vect{x}_{n_j, \vect{\lambda}}^*, \vect{x})) \biggr)
        \to  0
    \end{multline*}
    as \(j \to \infty\) almost surely.
    By \cref{thm:app-unif-convergence-cond-gps}, the sequence \(\vect{\lambda} \cdot (\vect{k}_{n, :, m}(\vect{u}, \vect{x}) - \vect{k}_{n, :, m}(\vect{x}, \vect{x})) \to \vect{\lambda} \cdot (\vect{k}_{\infty, :, m}(\vect{u}, \vect{x}) - \vect{k}_{\infty, :, m}(\vect{x}, \vect{x}))\) as \(n \to \infty\).
    We have shown that this limit is zero, since the limit of a subsequence is zero.
    This is Equation~\eqref{eq:app-kinfty-const-interim}.
    
    Hence, for all \(\vect{u} \in \mathcal{X}\), \(\vect{\lambda} \cdot \vect{k}_{\infty, :, m}(\vect{u}, \vect{x}) = \vect{\lambda} \cdot \vect{k}_{\infty, :, m}(\vect{x}, \vect{x})\).
    This also holds for all \(\vect{\lambda} \in \Lambda\) and \(m \in \{1, \dots, M\}\), so \(\matr{K}_\infty(\vect{u}, \vect{x}) = \matr{K}_\infty(\vect{x}, \vect{x})\). Furthermore, it holds for all \(\vect{x} \in \mathcal{X}\) and so, using the symmetry of \(\matr{K}_\infty\) in its arguments, we have,
    \[\forall \vect{x}, \vect{u}, \vect{x}', \vect{u}' \in \mathcal{X} \quad \mathbb{P}\bigl(\matr{K}_\infty(\vect{u}, \vect{x}) = \matr{K}_\infty(\vect{u}', \vect{x}')\bigr) = 1.\]
    
    To extend this to a statement about all choices of \(\vect{x}, \vect{u} \in \mathcal{X}\) simultaneously, observe that \(\mathcal{X}\) is a compact metric space and thus is separable. That is, it has a countable, dense subset \(\mathcal{X}' \subset \mathcal{X}\).
    By countable subadditivity of \(\mathbb{P}\), we have
    \[\mathbb{P}\bigl(\forall \vect{x}, \vect{u}, \vect{x}', \vect{u}' \in \mathcal{X}' \quad \matr{K}_\infty(\vect{u}, \vect{x}) = \matr{K}_\infty(\vect{u}', \vect{x}')\bigr) = 1.\]
    Then, by continuity of the sample paths of \(\matr{K}_\infty\), this extends to all of \(\mathcal{X}\),
    \[\mathbb{P}\bigl(\forall \vect{x}, \vect{u}, \vect{x}', \vect{u}' \in \mathcal{X} \quad \matr{K}_\infty(\vect{u}, \vect{x}) = \matr{K}_\infty(\vect{u}', \vect{x}')\bigr) = 1.\]
    That is, the sample paths of \(\matr{K}_\infty\) are almost surely constant.
\end{proof}

\begin{lemma} \label{thm:app-const-sample-paths-f}
    If \(\matr{K}_\infty\) has almost surely constant sample paths then \(\vect{f} - \vect{\mu}_\infty\) has almost surely constant sample paths.
\end{lemma}
\begin{proof}
Let \(\vect{u}, \vect{x} \in \mathcal{X}\). Then
\begin{multline*}
    \mathrm{Var}_\infty[(\vect{f}(\vect{u}) - \vect{\mu}_\infty(\vect{u})) - (\vect{f}(\vect{x}) - \vect{\mu}_\infty(\vect{x}))] \\
    = \mathrm{Var}_\infty[\vect{f}(\vect{u}) - \vect{f}(\vect{x})]
    = \underbrace{\matr{K}_\infty(\vect{u}, \vect{u}) - \matr{K}_\infty(\vect{x}, \vect{u})}_{0 \;\text{a.s.}} + \underbrace{\matr{K}_\infty(\vect{x}, \vect{x}) - \matr{K}_\infty(\vect{u}, \vect{x})}_{0 \;\text{a.s.}}
    = 0.
\end{multline*}
Noting that \(\mathbb{E}_\infty[(\vect{f}(\vect{u}) - \vect{\mu}_\infty(\vect{u})) - (\vect{f}(\vect{x}) - \vect{\mu}_\infty(\vect{x}))] = 0\) almost surely, we can extend this result to the unconditional variance using the law of total variance
\begin{align*}
    \mathrm{Var}[(\vect{f}(\vect{u}) - \vect{\mu}_\infty(\vect{u})&) - (\vect{f}(\vect{x}) - \vect{\mu}_\infty(\vect{x}))] \\
    &= \mathbb{E}\left[ \underbrace{\mathrm{Var}_\infty[(\vect{f}(\vect{u}) - \vect{\mu}_\infty(\vect{u})) - (\vect{f}(\vect{x}) - \vect{\mu}_\infty(\vect{x}))]}_{ = 0 \;\text{a.s.} } \right] \\
    &\quad+ \mathrm{Var}\left[ \underbrace{\mathbb{E}_\infty[(\vect{f}(\vect{u}) - \vect{\mu}_\infty(\vect{u})) - (\vect{f}(\vect{x}) - \vect{\mu}_\infty(\vect{x}))]}_{= 0 \;\text{a.s.}} \right] \\
    &= 0.
\end{align*}
Therefore, \((\vect{f}(\vect{u}) - \vect{\mu}_\infty(\vect{u})) - (\vect{f}(\vect{x}) - \vect{\mu}_\infty(\vect{x}))\) is a random vector with zero mean and a zero covariance matrix, and so \(\vect{f}(\vect{u}) - \vect{\mu}_\infty(\vect{u}) = \vect{f}(\vect{x}) - \vect{\mu}_\infty(\vect{x})\) almost surely.

For each pair \(\vect{x}, \vect{u} \in \mathcal{X}\), this holds with probability 1. To conclude that it holds with probability 1 for all \(\vect{x}, \vect{u} \in \mathcal{X}\) simultaneously, we use the same argument as at the end of \cref{thm:app-const-sample-paths-k}. For this, we need only observe that \(\mathcal{X}\) is separable and that the sample paths of \(\vect{f}\) and \(\vect{\mu}_\infty\) are continuous.
\end{proof}

Having shown that \(\vect{\mu}_\infty\) approximates \(\vect{f}\) up to a constant, we may now strengthen \cref{thm:app-cmokg-r-unif-zero-liminf}. This is not necessary for the proof of \cref{thm:app-consistency}, but does improve our understanding of the behavior of the MOKG.
\begin{corollary} \label{thm:app-cmokg-r-unif-zero}
    Suppose we select samples using C-MOKG with random scalarization weights \(\vect{\lambda}_1, \vect{\lambda}_2, \dots\) chosen independently according to distribution \(p(\vect{\lambda})\).
    Then,
    \[\forall \vect{\lambda} \in \Lambda \quad \mathbb{P}\left( \sup_{\vect{x} \in \mathcal{X},\, m \in [M]} \alpha_{\text{C-MOKG}}^n(\vect{x}, m;\, \vect{\lambda}, \vect{c}) \to 0 \quad\text{as}\quad n \to \infty \right) = 1.\]
\end{corollary}
\begin{proof}
    Let \(\vect{\lambda} \in \Lambda\).
    By L\'evy's zero-one law,
    \[\mathbb{E}_n \Bigl[\max_{\vect{x} \in \mathcal{X}} \vect{\lambda} \cdot \vect{f}(\vect{x}) \Bigr] \to \mathbb{E}_\infty \Bigl[\max_{\vect{x} \in \mathcal{X}} \vect{\lambda} \cdot \vect{f}(\vect{x}) \Bigr]\]
    as \(n \to \infty\) almost surely.
    Further, we showed in \cref{thm:app-unif-convergence-cond-gps} that \(\vect{\mu}_n \to \vect{\mu}_\infty\) uniformly as \(n \to \infty\) almost surely.
    Therefore, the residual uncertainty also converges,
    \begin{multline*}
        H_n(\vect{\lambda}) = \mathbb{E}_n \!\left[\max_{\vect{x} \in \mathcal{X}} \vect{\lambda} \cdot \vect{f}(\vect{x})\right] - \max_{\vect{x} \in \mathcal{X}} \vect{\lambda} \cdot \vect{\mu}_n(\vect{x}) \\
        \to \mathbb{E}_\infty \!\left[\max _{\vect{x} \in \mathcal{X}}\vect{\lambda} \cdot \vect{f}(\vect{x})\right] - \max_{\vect{x} \in \mathcal{X}} \vect{\lambda} \cdot \vect{\mu}_\infty(\vect{x}) = H_\infty(\vect{\lambda})
    \end{multline*}
    as \(n \to \infty\) almost surely.
    However, from \cref{thm:app-const-sample-paths-f}, the sample paths of \(\vect{f} - \vect{\mu}_\infty\) are almost surely constant and so \(\max_{\vect{x} \in \mathcal{X}} \vect{\lambda} \cdot \vect{f}(\vect{x}) = \vect{\lambda} \cdot \vect{f}(\vect{x}_{\infty, \vect{\lambda}}^*)\) almost surely.
    Hence,
    \[H_n(\vect{\lambda}) \to H_\infty(\vect{\lambda}) = \mathbb{E}_\infty[\vect{\lambda} \cdot \vect{f}(\vect{x}_{\infty, \vect{\lambda}}^*)] - \vect{\lambda} \cdot \vect{\mu}_\infty (\vect{x}_{\infty, \vect{\lambda}}^*) = 0 \quad\text{(a.s.)}.\]
    Finally, we note that by \cref{thm:app-hn-bigger-than-cmokg},
    \[0 \leq \sup_{\vect{x} \in \mathcal{X},\, m \in [M]} \alpha_{\text{C-MOKG}}^n(\vect{x}, m;\, \vect{\lambda}, \vect{c}) \leq H_n(\vect{\lambda}) \to 0\]
    and so \(\sup_{\vect{x} \in \mathcal{X},\, m \in [M]} \alpha_{\text{C-MOKG}}^n(\vect{x}, m;\, \vect{\lambda}, \vect{c}) \to 0\) as \(n \to \infty\) almost surely.
\end{proof}

Following \cref{thm:app-const-sample-paths-f}, we can now also prove \cref{thm:app-consistency}.
\begin{proof}[Proof of \cref{thm:app-consistency}]
Let \(\vect{\lambda} \in \Lambda\).
Recall that in Equation~\eqref{eq:app-cmokg-recommendations}, for each \(n\), we have defined \(\vect{x}_{n, \vect{\lambda}}^* \in \argmax \mathbb{E}_n[\vect{\lambda} \cdot \vect{f}]\) to be a random variable which maximizes the posterior mean of the scalarized objective at stage \(n\). Similarly, let \(\vect{x}_{\vect{\lambda}}^* \in \argmax \vect{\lambda} \cdot \vect{f}\) be a random variable maximizing the scalarized objective \(\vect{\lambda} \cdot \vect{f}\). Note that these choices are not necessarily unique.

We wish to first show that \(\vect{\lambda} \cdot \vect{f}(\vect{x}_{n, \vect{\lambda}}^*) \to \vect{\lambda} \cdot \vect{f}(\vect{x}_{\vect{\lambda}}^*)\) almost surely.
By definition of \(\vect{x}_{\vect{\lambda}}^*\), the quantity \(\vect{\lambda} \cdot \vect{f}(\vect{x}_{\vect{\lambda}}^*) - \vect{\lambda} \cdot \vect{f}(\vect{x}_{n, \vect{\lambda}}^*)\) is non-negative for each \(n\).
Therefore, it suffices to prove that, almost surely, \(\limsup_{n \to \infty} \vect{\lambda} \cdot \vect{f}(\vect{x}_{\vect{\lambda}}^*) - \vect{\lambda} \cdot \vect{f}(\vect{x}_{n, \vect{\lambda}}^*) \leq 0\).

Observe that we may split the limit superior to give
\begin{align*}
    \limsup_{n \to \infty} &\; \vect{\lambda} \cdot \vect{f}(\vect{x}_{\vect{\lambda}}^*) - \vect{\lambda} \cdot \vect{f}(\vect{x}_{n, \vect{\lambda}}^*) \\
    \leq& \phantom{+}\limsup_{n \to \infty}\;
        \Bigl(\vect{\lambda} \cdot \vect{f}(\vect{x}_{\vect{\lambda}}^*) - \vect{\lambda} \cdot \vect{\mu}_\infty(\vect{x}_{\vect{\lambda}}^*)\Bigr)
        - \Bigl(\vect{\lambda} \cdot \vect{f}(\vect{x}_{n, \vect{\lambda}}^*) - \vect{\lambda} \cdot \vect{\mu}_\infty(\vect{x}_{n, \vect{\lambda}}^*)\Bigr) \\
    & + \limsup_{n \to \infty}\;
        \Bigl(\vect{\lambda} \cdot \vect{\mu}_\infty(\vect{x}_{\vect{\lambda}}^*) - \vect{\lambda} \cdot \vect{\mu}_n(\vect{x}_{\vect{\lambda}}^*)\Bigl)
        - \Bigl(\vect{\lambda} \cdot \vect{\mu}_\infty(\vect{x}_{n, \vect{\lambda}}^*) - \vect{\lambda} \cdot \vect{\mu}_n(\vect{x}_{n, \vect{\lambda}}^*)\Bigr) \\
    & + \limsup_{n \to \infty}\; \vect{\lambda} \cdot \vect{\mu}_n(\vect{x}_{\vect{\lambda}}^*) - \vect{\lambda} \cdot \vect{\mu}_n(\vect{x}_{n, \vect{\lambda}}^*).
\end{align*}
By \cref{thm:app-const-sample-paths-k,thm:app-const-sample-paths-f}, the sample paths of \(\vect{\lambda} \cdot \vect{f} - \vect{\lambda} \cdot \vect{\mu}_\infty\) are almost surely constant and so the first line above is almost surely zero.
The middle line is also almost surely zero since \(\vect{\mu}_n \to \vect{\mu}_\infty\) uniformly almost surely by \cref{thm:app-unif-convergence-cond-gps}.
Finally, the bottom line is at most zero by definition of the \(\vect{x}_{n, \vect{\lambda}}^*\).
Therefore
\[\limsup_{n \to \infty} \vect{\lambda} \cdot \vect{f}(\vect{x}_{\vect{\lambda}}^*) - \vect{\lambda} \cdot \vect{f}(\vect{x}_{n, \vect{\lambda}}^*) \leq 0\]
as desired and we conclude that \(\vect{\lambda} \cdot \vect{f}(\vect{x}_{n, \vect{\lambda}}^*) \to \vect{\lambda} \cdot \vect{f}(\vect{x}_{\vect{\lambda}}^*)\) almost surely.

Convergence in mean follows from the Dominated Convergence Theorem with dominating variable \(\sup_{\vect{x}} |\vect{\lambda} \cdot \vect{f}(\vect{x})|\).
\end{proof}

\newpage
\section{Further Experimental Details}
\label{sec:experiment-details}
This section contains further technical details of the experiments described in the main text.

\subsection{Test Problems}
Two families of test problem are presented, \(\vect{f}^*: [0, 1]^2 \to \mathbb{R}^2\). They are generated by sampling a GP for each objective for an example). In the first family, the objectives differ in their length scale, while in the second family they differ in the presence of additive, Gaussian observation noise. \cref{tab:app-test-problems} shows the hyper-parameters used for both GPs. \cref{fig:test-problem-lengthscales} and \cref{fig:test-problem-observationnoise} below show examples of a test problem from the first and second families respectively.

When we sample the GP, we do not know the future locations at which the BO algorithm will try to sample it. To get around this, we sample each GP at 100 points, distributed over \([0, 1]^2\) using a scrambled, 2-dimensional Sobol' sequence. We then use the posterior mean of the GP after conditioning on these points, as the test function.

\begin{table}[h]
    \caption{Hyper-parameters for the Gaussian processes used to generate the two families of test problem.}
    \label{tab:app-test-problems}
    \vskip 0.15in
    \begin{center}
    \begin{small}
    \begin{sc}
    \begin{tabular}{lcccc}
    \toprule
        & \multicolumn{2}{c}{Family 1} & \multicolumn{2}{c}{Family 2} \\
        \cmidrule(lr){2-3}
        \cmidrule(lr){4-5}
        & Objective 1 & Objective 2 & Objective 1 & Objective 2 \\
    \midrule
        Kernel & Mat\'ern-\(5/2\) & Mat\'ern-\(5/2\) & Mat\'ern-\(5/2\) & Mat\'ern-\(5/2\) \\
        Isotropic length scale & 0.2 & 1.8 & 0.4 & 0.4 \\
        Output scale & 1 & 50 & 1 & 1 \\
        Constant mean & 0 & 0 & 0 & 0 \\
        Noise standard deviation & 0 & 0 & 1 & 0 \\
    \bottomrule
    \end{tabular}
    \end{sc}
    \end{small}
    \end{center}
    \vskip -0.1in
\end{table}

\begin{figure*}[ht]
    \centering
    \begin{subfigure}[b]{0.66\textwidth}
        \centering
        \includegraphics[width=\textwidth]{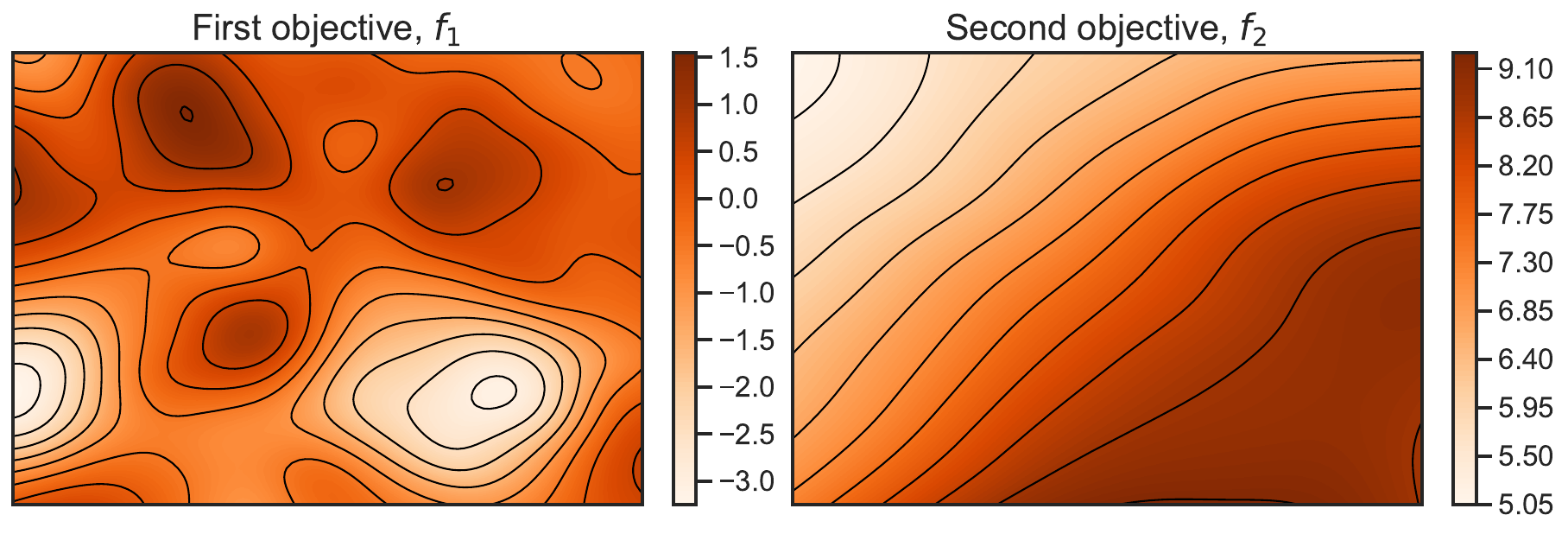}
        \caption{Contours}
        \label{fig:test-problem-lengthscales-contour}
    \end{subfigure}
    \begin{subfigure}[b]{0.33\textwidth}
        \centering
        \includegraphics[width=\textwidth]{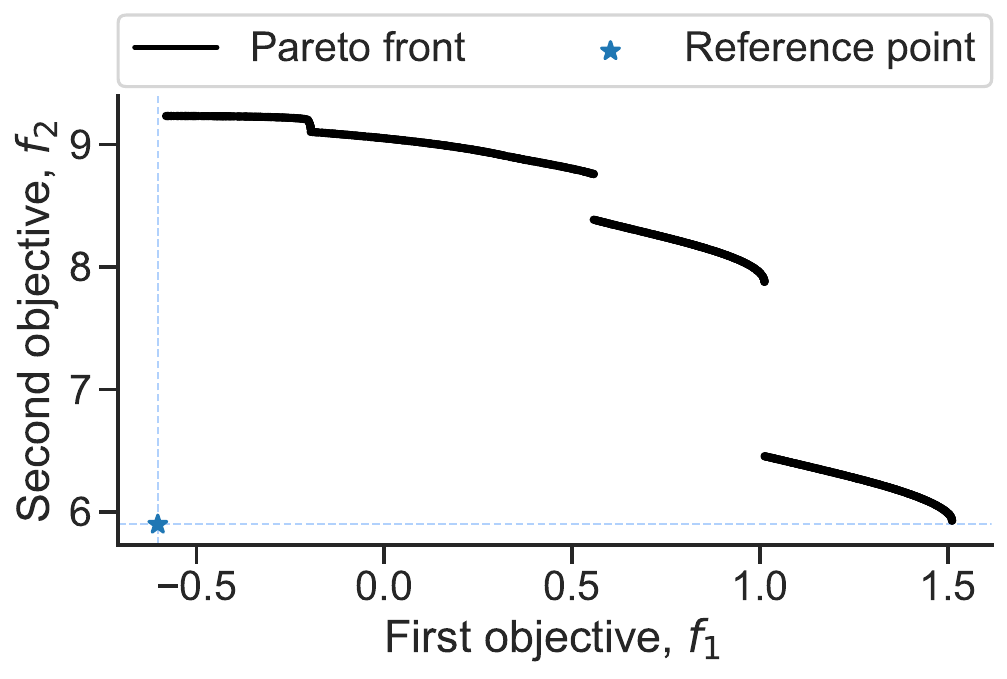}
        \caption{Pareto front}
        \label{fig:test-problem-lengthscales-pfront}
    \end{subfigure}
    \caption{An example from the first set of test problems generated as samples of a GP. Panel~(\subref*{fig:test-problem-lengthscales-contour}) shows how the first objective is made harder to learn by giving the GP a shorter length scale.
    Panel~(\subref*{fig:test-problem-lengthscales-pfront}) shows the Pareto front of the test problem. 
    }
    \label{fig:test-problem-lengthscales}
\end{figure*}

\begin{figure*}[ht]
    \vskip 0.2in
    \begin{center}
    \begin{subfigure}[b]{0.66\textwidth}
        \centering
        \includegraphics[width=\textwidth]{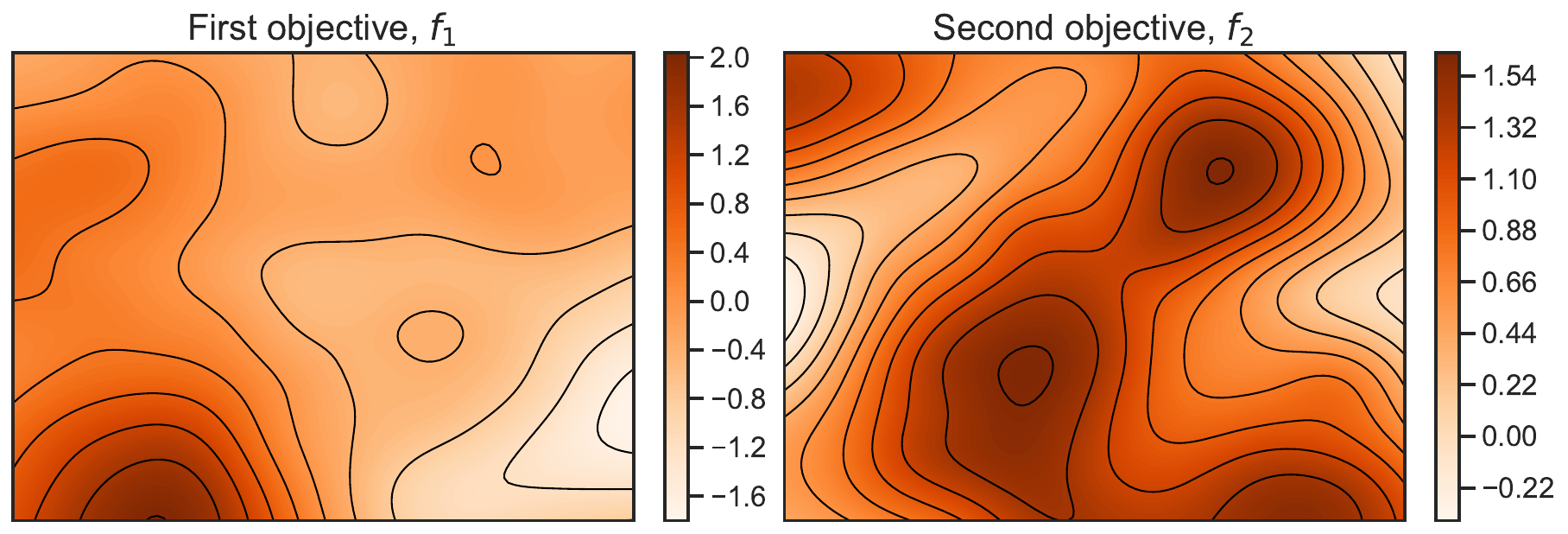}
        \caption{Contours}
        \label{fig:test-problem-observationnoise-contour}
    \end{subfigure}
    \begin{subfigure}[b]{0.33\textwidth}
        \centering
        \includegraphics[width=\textwidth]{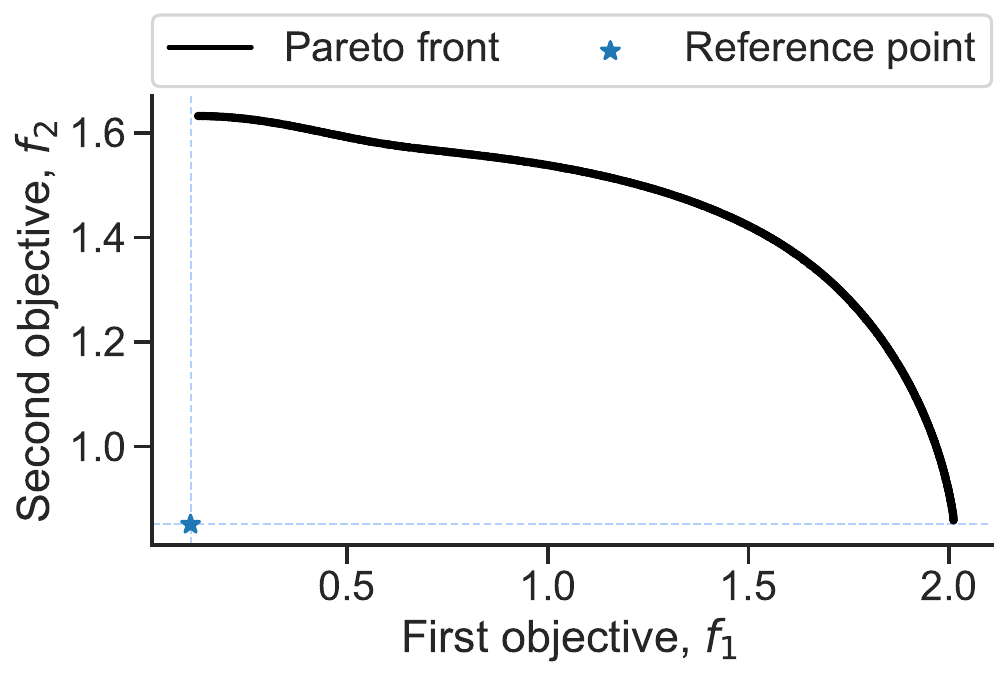}
        \caption{Pareto front}
        \label{fig:test-problem-observationnoise-pfront}
    \end{subfigure}
    \caption{An example from the second set of test problems generated as samples of a GP. Panel~(\subref*{fig:test-problem-observationnoise-contour}) shows that both objectives have the same length scale. Instead, the first objective is made harder-to-learn by adding observation noise when sampling.
    Panel~(\subref*{fig:test-problem-observationnoise-pfront}) shows the Pareto front of the test problem.
    }
    \label{fig:test-problem-observationnoise}
    \end{center}
    \vskip -0.2in
\end{figure*}

\subsection{Surrogate Model}
To model the objectives during the Bayesian optimization, we use independent GPs for the objectives. We use a Mat\'ern-\(5/2\) kernel and at each step we fit the hyper-parameters to the data observed so far using maximum a posteriori (MAP) estimates.
Before fitting the GP, the observations are standardized to have zero mean and unit variance. This is a common trick used for real-world problems where the prior distributions must be specified based on the data. The reverse transformation is applied when making predictions with the GP.

\paragraph{Prior mean} The prior mean of the GP is modeled as a constant function. Bayesian optimization seeks to sample points with large objective values and so fitting the prior mean to this biased sample would introduce a bias to the inferred value. Therefore, the mean was fitted as part of the MAP estimate on the initial six data points, but then held fixed at this value for all future iterations while the other hyper-parameters were fitted.

\paragraph{Observation noise} The observation model in Equation~\eqref{eq:obs-model-decoupled} contains additive Gaussian observation noise. The variance for this observation noise is only fitted for the first objective in the second family of test problems, since this is the only objective where noise is added. In the first family of test problems and for the second objective of the second family, the variance of the noise was fixed at a negligible value of \(10^{-4}\). It was not set to zero since this leads to numerical instability. Usually experimenters know whether their problem is stochastic or deterministic, so we view this as the most natural way to model the objectives.

\paragraph{Prior distributions}
For the length scales, output scales and observation noise variance (where it is fitted), a Gamma prior distribution is used. The Gamma distribution is parameterized by a shape parameter \(\alpha\) and a rate parameter \(\beta\). It is supported on \((0, \infty)\) and has probability density function given by
\begin{equation} \label{eq:app-gamma-pdf}
    p(z) = \frac{\beta^\alpha}{\Gamma(\alpha)}z^{\alpha - 1} e^{-\beta z}
\end{equation}
where \(\Gamma\) is the gamma function.

The GP prior mean is modeled as a constant function. A uniform distribution across the whole real line is used as an improper prior for this constant. This is implemented by simply not adding any contribution from the prior distribution on the GP prior mean to the marginal log-likelihood.

As explained in the main text, information on which objective has the shorter length scale is included in the prior distributions used for the first family of test problems. Conversely, the priors on the length scales of the objectives in the second problem are identical. \cref{fig:app-lengthscale-priors--family-1,fig:app-lengthscale-priors--family-2} show the prior distributions on the length scales used to model the two families of test problems.

Prior distributions for all the hyper-parameters are summarized in \cref{tab:app-surrogate-hyperparameters--family-1,tab:app-surrogate-hyperparameters--family-2}.

\begin{figure}[h]
    \centering
    \includegraphics[width=0.7\textwidth]{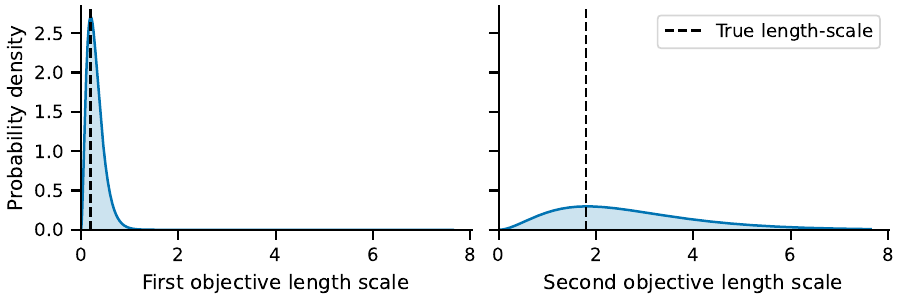}
    \caption{Prior distributions on the length scales for the surrogate model used for the first family of test problems. These encode prior information that the first objective has a shorter length scale.}
    \label{fig:app-lengthscale-priors--family-1}
\end{figure}
\begin{figure}[h]
    \centering
    \includegraphics[width=0.7\textwidth]{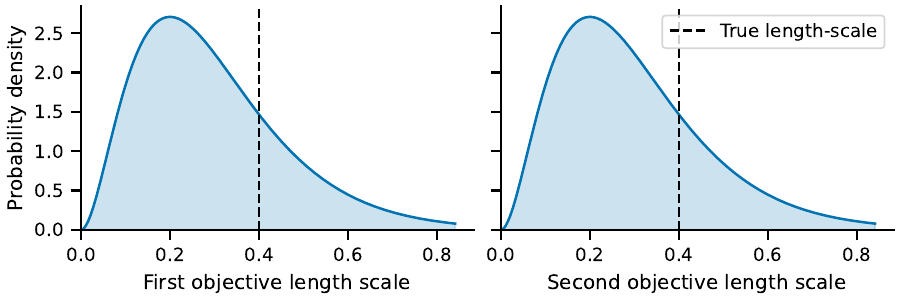}
    \caption{Prior distributions on the length scales for the surrogate model used for the second family of test problems. The same prior is used for both objectives.}
    \label{fig:app-lengthscale-priors--family-2}
\end{figure}

\begin{table}[h!]
    \caption{Prior distributions used for the hyper-parameters of the surrogate model for the standardized data in the first family of test problems.}
    \label{tab:app-surrogate-hyperparameters--family-1}
    \vskip 0.15in
    \begin{center}
    \begin{small}
    \begin{sc}
    \begin{tabular}{lcc}
    \toprule
        & Objective 1 & Objective 2 \\
    \midrule
        Kernel & Mat\'ern-\(5/2\) & Mat\'ern-\(5/2\) \\
        Isotropic length scale & \(\mathrm{Gamma}(\alpha=3, \beta=10)\) & \(\mathrm{Gamma}(\alpha=3, \beta=1.1)\) \\
        Output scale & \(\mathrm{Gamma}(\alpha=2, \beta=0.15)\) & \(\mathrm{Gamma}(\alpha=2, \beta=0.15)\) \\
        Constant mean & No prior & No prior \\
        Noise variance & Fixed at \(10^{-4}\) & Fixed at \(10^{-4}\) \\
    \bottomrule
    \end{tabular}
    \end{sc}
    \end{small}
    \end{center}
    \vskip -0.1in
\end{table}

\begin{table}[h!]
    \caption{Prior distributions used for the hyper-parameters of the surrogate model for the standardized data in the second family of test problems.}
    \label{tab:app-surrogate-hyperparameters--family-2}
    \vskip 0.15in
    \begin{center}
    \begin{small}
    \begin{sc}
    \begin{tabular}{lcc}
    \toprule
        & Objective 1 & Objective 2 \\
    \midrule
        Kernel & Mat\'ern-\(5/2\) & Mat\'ern-\(5/2\) \\
        Isotropic length scale & \(\mathrm{Gamma}(\alpha=3, \beta=10)\) & \(\mathrm{Gamma}(\alpha=3, \beta=10)\) \\
        Output scale & \(\mathrm{Gamma}(\alpha=2, \beta=0.15)\) & \(\mathrm{Gamma}(\alpha=2, \beta=0.15)\) \\
        Constant mean & No prior & No prior \\
        Noise variance & \(\mathrm{Gamma}(\alpha=1.1, \beta=0.05)\) & Fixed at \(10^{-4}\) \\
    \bottomrule
    \end{tabular}
    \end{sc}
    \end{small}
    \end{center}
    \vskip -0.1in
\end{table}

\subsection{Random Numbers}
When comparing C-MOKG to the benchmark algorithm, we tried to keep as much as possible the same throughout the experiment. Consequently,
\begin{itemize}
    \item the algorithms were tested on the same 100 test problems in each family;
    \item each test problem was assigned a different set of six initial points, which were used for both algorithms;
    \item the same sequence of points were used in the Monte-Carlo approximations to approximate the expectation over \(\vect{\lambda}\) in both algorithms.
\end{itemize}

\end{document}